\documentclass{article}

\usepackage[english]{babel}
\usepackage{amsthm}
\usepackage{arxiv-style}
\usepackage{mathtools}
\usepackage{xfrac} 
\usepackage[utf8]{inputenc} 
\usepackage[T1]{fontenc}    
\usepackage{url}            
\usepackage{booktabs}    
\usepackage{amsfonts}      
\usepackage{nicefrac}      
\usepackage{microtype}    
\usepackage{lipsum}
\usepackage{fancyhdr}      
\usepackage{graphicx}       
\usepackage[ruled,vlined]{algorithm2e}
\usepackage{stackrel}

\usepackage{caption}
\usepackage{subcaption}

\usepackage{appendix}
\usepackage{csquotes}
\usepackage{amssymb}
\usepackage{sectsty}
\usepackage{accents}
\usepackage{comment}
\usepackage{mdframed}
\usepackage{tcolorbox}
\tcbuselibrary{listings, breakable}
\usepackage{array}

\usepackage{enumitem}
\usepackage[dvipsnames]{xcolor}
\usepackage[numbers,sort&compress]{natbib}

\usepackage[colorlinks=true,
             linkcolor=blue,
             citecolor=blue,
             urlcolor=blue]{hyperref}

\hypersetup{
    colorlinks=true,    
    linkcolor=blue,     
    citecolor=magenta,  
    urlcolor=blue       
}

\def\R{\mathbb{R}}

\def\M{\mathcal{M}}
\def\N{\mathbb{N}}

\def\F{\mathcal{F}}

\def\Ci{\mathcal{C}}

\def\Int{\text{Int}}

\def\wto{\rightharpoonup}

\def\<{\langle}
\def\>{\rangle}
\def\|{\vert}
\def\!{\Vert}

\def\d{\, \mathrm{d}}

\def\X{\mathcal{X}}
\def\Y{\mathcal{Y}}
\def\Ri{\mathcal{R}}

\def\Ci{\mathcal{C}}

\def\bphi{\boldsymbol\varphi}

\DeclareMathOperator*{\argmin}{arg\,min}

\DeclareMathOperator{\cosre}{\mathsf{C}\Ri\textsf{UOT}(\alpha,\beta)}
\DeclareMathOperator{\cosree}{\mathsf{C}\Ri\textsf{UOT}_{\varepsilon_{n}}(\alpha,\beta)}

\newcommand{\D}[3]{\ensuremath{\mathrm{D}_{#1}(#2 \, \vert \, #3)}}

\newcommand{\norm}[1]{\left\lVert#1\right\rVert}

\expandafter\let\expandafter\oldproof\csname\string\proof\endcsname
\let\oldendproof\endproof
\renewenvironment{proof}[1][\proofname]{%
  \oldproof[\normalfont\sffamily #1]%
}{\oldendproof}

\newtheoremstyle{mystyle}
  {0.5 cm}
  {\topsep}
  {\itshape}
  {}
  {\bfseries}
  {:}
  {3mm}
   {\thmname{#1}\thmnumber{ #2}\thmnote{ \normalfont(#3)}}

\newtheoremstyle{mystyleBox}
  {0.5 cm}
  {0.5 cm}
  {\normalfont}
  {}
  {\bfseries}
  {.}
  {3mm}
  {\thmnote{\bfseries#3}}

  \newtheoremstyle{mystyleNormalFontB}
  {0.5 cm}
  {0.5 cm}
  {\normalfont}
  {}
  {\bfseries}
  {:}
  {3mm}
  {\thmname{#1}\thmnumber{ #2}\thmnote{ (#3)}}

\theoremstyle{mystyleNormalFontB}
\newtheorem{theorem}{Theorem}[section]
\theoremstyle{mystyleNormalFontB}
\newtheorem{coroll}[theorem]{Corollary}
\theoremstyle{mystyleNormalFontB}
\newtheorem{lemma}[theorem]{Lemma}
\theoremstyle{mystyleNormalFontB}
\newtheorem{prop}[theorem]{Proposition}
\theoremstyle{mystyleNormalFontB}
\newtheorem{defn}[theorem]{Definition}
\theoremstyle{mystyleNormalFontB}
\newtheorem{oss}[theorem]{Remark}
\theoremstyle{mystyleNormalFontB}

\theoremstyle{mystyleNormalFontB}
\newtheorem{problem}[theorem]{Problem}
\theoremstyle{mystyleBox}

\theoremstyle{mystyleNormalFontB}

\title{Structured Matching via Cost-Regularized Unbalanced Optimal Transport}
\author{
  Emanuele Pardini\\
  University of Pisa \\
  Pisa\\
  \texttt{e.pardini21@studenti.unipi.it} 
    \And
  Katerina Papagiannouli\\
  University of Pisa \\
  Pisa\\
  \texttt{aikaterini.papagiannouli@unipi.it} \\
}

\begin{document}

\maketitle

\begin{abstract}
    \sloppy
Unbalanced optimal transport (UOT) provides a flexible way to match or compare nonnegative finite Radon measures. However, UOT requires a predefined ground transport cost, which may misrepresent the data’s underlying geometry. Choosing such a cost is particularly challenging when datasets live in heterogeneous spaces, often motivating practitioners to adopt Gromov–Wasserstein formulations. To address this challenge, we introduce cost-regularized unbalanced optimal transport (CR-UOT), a framework that allows the ground cost to vary while allowing mass creation and removal. We show that CR-UOT incorporates unbalanced Gromov–Wasserstein–type problems through families of inner-product costs parameterized by linear transformations, enabling the matching of measures (or point clouds) across Euclidean spaces. We develop algorithms for such CR-UOT problems using entropic regularization and demonstrate that this approach improves the alignment of heterogeneous single-cell omics profiles, especially when many cells lack direct matches.\end{abstract}

\section{Introduction}

Optimal Transport (OT) has become a central tool in machine learning and related fields, providing a principled way to compare probability measures while respecting geometric structure. OT has been successfully applied to generative modeling \citep{montavon2016wasserstein, genevay2018learning, tong2023improving, brechet2023critical}, adversarial training \citep{sinha2018certifying, wong2019wasserstein}, domain adaptation \citep{courty2017joint, fatras2021unbalanced}, neuroscience \citep{janati2020multi}, and single-cell biology \citep{schiebinger2019optimal, bunne2023learning, demetci2022scot, demetci2022scotv2}. In recent years, OT has become a powerful framework for addressing the graph matching problem. It is often viewed as a continuous relaxation of the Quadratic Assignment Problem, formulated through the Gromov-Wasserstein (GW) distance \citep{montavon2016wasserstein, xu2019gromov}, which extends the classical Wasserstein distance to compare distributions defined on different metric spaces. Several variants have been developed to handle labeled graphs, such as the Fused Gromov-Wasserstein (FGW) distance \citep{titouan2019optimal}. 

Despite its successes, classical OT rests on two restrictive assumptions: (i) perfect mass preservation, and (ii) a fixed ground cost function. These assumptions often break down in applications. First, the \emph{mass preservation} assumption requires that the marginals of the transport plan exactly match the input measures. This is unrealistic when data are noisy, incomplete, or inherently heterogeneous. To address this limitation, the framework of unbalanced optimal transport (UOT) has been developed, which relaxes the strict mass conservation constraint and allows for comparisons between measures of different total mass \citep{Kondratyev2016, Liero2018, Chizat2018, demetci2022scotv2}. A notable benefit of UOT is its robustness to outliers, since unmatched mass can be discarded rather than transported. This property has made UOT valuable in diverse applications, including deep learning theory \citep{Chizat2018, Rotskoff2019,mazelet2025unsupervisedlearningoptimaltransport}, single-cell biology \citep{Schiebinger2019, demetci2022scot, demetci2022scotv2}, and domain adaptation \citep{Fatras2021}. Unbalanced GW \citep{sejourne2019sinkhorn} and Fused Unbalanced GW (FUGW) \citep{thual2022aligning} were proposed to generalize the GW and FGW distances to unbalanced settings with application in positive unlabeled learning and brain alignment. 

Second, OT requires specifying a \emph{ground cost} that quantifies discrepancies between source and target points. In many applications---especially when measures lie in different ambient spaces or dimensions---this cost is unknown or may misrepresent the true geometry of the problem. Gromov--Wasserstein (GW) distances \citep{memoli2011gromov} circumvent this issue by aligning distributions through the comparison of relational structures within each space. While elegant, GW poses two difficulties: it results in a nonconvex quadratic optimization problem with high computational cost, and it forfeits some of the interpretability and guarantees available in linear OT settings \citep{vayer2020contribution, dumont2022existence}.  

Despite its flexibility, computing UOT remains computationally demanding: it requires solving a linear program whose complexity scales cubically with the number of samples \citep{Pele2009, peyre2019computational}. Moreover, empirical estimation of UOT distances is challenging due to the curse of dimensionality \citep{Dudley1969}. To mitigate these issues, several tractable variants have been proposed with reduced complexity and improved statistical properties, such as entropic OT \citep{Cuturi2013, Pham2020}, minibatch OT \citep{Fatras2020, Fatras2021}, sliced UOT \citep{bonet2024slicing}, and cost regularized OT \citep{sebbouh2024structured}.

In this work, we focus on extending the ideas of  cost regularized OT to develop UOT methods.\emph{This work addresses both challenges--unbalancedness and unknown ground cost--simultaneously and at the same time time computational efficiency.} We introduce \emph{cost-regularized unbalanced OT (CR-UOT)} inspired by \citep{sebbouh2024structured}, a framework that allows the ground cost to vary while relaxing marginal constraints. Our approach unifies and extends UOT and certain GW problems by introducing convex regularizers over costs, yielding families of parameterized linear inner-product costs across spaces. This formulation admits efficient algorithms via entropic regularization, while retaining connections to Monge maps through theoretical grounded approximation and convergent results. The proofs of the results are given in the Appendix.

\paragraph{Contributions.} Our main contributions are:  
\begin{itemize}[leftmargin=*]
    \item \textbf{Formulation:} We introduce CR-UOT, a framework combining convex cost regularization with unbalanced OT, unifying and extending existing UOT and GW formulations.
    \item \textbf{Theory:} We prove existence of minimizers and establish convergence results of values and minimizers of the entropic regularized problem. Focusing on inner-product costs parameterized by linear transformations across spaces, we introduce a simple block coordinate descent algorithm to solve the associated CR-UOT problem. We show that, under mild conditions, optimal couplings are induced by deterministic Monge maps. We introduce entropic unbalnced Monge maps across spaces and show that they converge to the ground truth Monge maps under suitable assumptions.
    \item \textbf{Applications:} We demonstrate that the use of such entropic maps improves alignment of heterogeneous single-cell multiomics datasets, particularly when modalities lack direct correspondence or differ in proportions across cell types similar to \citep{demetci2022scot, demetci2022scotv2}.
\end{itemize}

\section{Background}\label{sec:intro}
\paragraph{Notations.}
In what follows, we consider $\mathcal{X}$ and $\mathcal{Y}$ to be compact metric spaces, and $\alpha \in \mathcal{M}^{+}(\mathcal{X})$  $\beta\in \M(\Y)^{+}$ to be finite positive Radon measures satisfying $m(\alpha)m(\beta) \neq 0$, where \(m(\mu)\) is the total mass of the measure \(\mu\). \(\Ci(\X), \, \Ci_b(\X)\) are continuous functions and bounded continuous functions on \(\X\) respectively. Given $\pi \in \M^{+}(\X \times \Y)$, we define its marginals $\pi_{i} = p_{i_{\#}}\pi$ for $i=1,2$.

\paragraph{Optimal Transport Problem ($\mathsf{OT}$).}
Given a family of all possible \emph{couplings} between $\alpha$ and $\beta$ 
\[
\Pi(\alpha,\beta)
= \left\{ \pi \in \mathcal{M}^{+}(\mathcal{X}\times\mathcal{Y})
:\; \pi_1= \alpha,\; \pi_2 = \beta \right\},
\]
we denote the linear OT cost between $\alpha$ and $\beta$ with cost $c\in\mathcal{C}(\mathcal{X}\times\mathcal{Y})$ as
\[
\mathsf{OT}(\alpha, \beta)
\triangleq
\min_{\pi\in\Pi(\alpha,\beta)}
\int_{\mathcal{X}\times\mathcal{Y}} c(x,y)\,\mathrm{d}\pi(x,y),
\]
which is a \emph{linear} problem in $\pi$, see \citep{santambrogio2015optimal}. When $c=d_{\mathcal{X}}^{\,p}$ and
$\mathcal{X}=\mathcal{Y}$, the $\mathsf{OT}$ defines a distance between
probability measures for all $p\ge 1$, see \citep{villani2008optimal}.

\paragraph{Unbalanced Optimal Transport (\textsf{UOT)}}
We recall the static formulation of UOT proposed by \citep{liero2018optimal}, which uses $\varphi$-divergence as penalty terms.
\begin{defn}[$\varphi$-divergences]
Let $\alpha, \beta\in \M^{+}(\X)$. Let $\varphi:[0,+\infty) \to [0,\infty]$ be an entropy function, i.e. $\varphi$ is convex and lower semicontinuous (lsc)
and \(\varphi(1)=0\). Denote \(\mathrm{dom}(\varphi) \triangleq \{ x \in [0,+\infty) \, \vert \, \varphi(x) < +\infty \} \subset [0,+\infty)\) 
\[
\varphi'_{\infty} \triangleq \lim_{x \to +\infty} \frac{\varphi(x)}{x}.
\]
The \(\varphi\)-divergence between \(\alpha\) and \(\beta\) is
\[
\mathrm{D}_{\varphi}(\alpha \| \beta) \triangleq 
\int_{\mathbb{R}^d} \varphi\left(\frac{d\alpha}{d\beta}(x)\right) d\beta(x) 
+ \varphi'_{\infty} \int_{\mathbb{R}^d} d\alpha^{\perp}(x),
\]
where \(\alpha^{\perp}\) is defined as 
\(\alpha = \tfrac{d\alpha}{d\beta}\beta + \alpha^{\perp}\) in Lebesgue decomposition form. 
We call $\varphi$ superlinear when $\varphi_{\infty}' = +\infty$.
\end{defn}

A special case of $\varphi$-divergence, which we use later in the experiments and the formulation of entropic regularizers, is the Kullback-Leibler (KL) divergence:
$$\mathrm{D}_{\mathrm{KL}}(\alpha \| \beta) = \begin{cases}\int_{\X} \varphi_{\mathrm{KL}}\left(\frac{d\alpha}{d\beta}\right) d \beta & \text{if $\alpha \ll \beta$} \\ +\infty & \text{otherwise,}\end{cases}$$
where $\varphi_{\mathrm{KL}}(x) = x \log(x) - x + 1$.
\begin{problem}[\textsf{UOT}]
For a lsc cost $c:\X\times \Y\to\mathbb{R}$ and entropy functions $\varphi_1,\varphi_2$, we denote
the unbalanced OT problem between $\alpha$ and $\beta$ as
\begin{equation*}
    \mathsf{UOT}(\alpha, \beta) \triangleq \inf_{\pi \in \M^{+}(\X\times \Y)}\int c(x,y)\, d\pi(x,y)
 \mathrm{D}_{\varphi_1}(\pi_1 \mid \alpha)
 + \mathrm{D}_{\varphi_2}(\pi_2 \mid \beta).
\end{equation*}
\end{problem}

Observe that when $\varphi_1=\varphi_2=\iota_{\{1\}}$, with $\iota_{\{1\}}(1)=0$ and $\iota_{\{1\}}(s)=+\infty$ for every $s\neq 1$, we find again the balanced $\mathsf{OT}(\alpha, \beta)$ problem.

In the results that follow we will usually assume one of the following compatibility conditions:
\begin{equation}\label{comp_cond}
    \left( m(\alpha) \mathrm{dom}(\varphi_1) \right) \cap \left( m(\beta) \mathrm{dom}(\varphi_2) \right)\neq \emptyset
\end{equation}
and the stronger one 
\begin{equation}\label{comp_cond_strong}
\left[ \mathrm{Int} (m(\alpha) \mathrm{dom}(\varphi_1)) \cap \left( m(\beta) \mathrm{dom}(\varphi_2) \right) \right]\cup \left[ \left( m(\alpha) \mathrm{dom}(\varphi_1) \right) \cap \Int(m(\beta) \mathrm{dom}(\varphi_2)) \right] \neq \emptyset.
    \end{equation}

\section{Cost-Regularized UOT}\label{sec:cost-reg-uot}

We now allow the cost itself to vary under a convex regularizer $\Ri$.

\begin{problem}[\textbf{\(\mathsf{C}\Ri\mathsf{UOT}\)}]\label{ROT}
Suppose we are given two entropy functions \(\varphi_1, \varphi_2 : [0,+\infty) \to [0,+\infty]\) and a convex function \(\Ri : \Ci(\X\times \Y) \to \R\cup\{+\infty\}\). We define the \emph{\(\Ri\)-regularized \(\bphi\)-unbalanced optimal transport problem} as
\begin{equation*}
    \mathsf{C}\Ri\textsf{UOT}(\alpha,\beta) \triangleq \inf_{\pi, c} \int_{\X\times \Y} c(x,y)\d \pi(x,y )+ \D{\varphi_1}{\pi_1}{\alpha} 
        + \D{\varphi_2}{\pi_2}{\beta} + \Ri(c).
\end{equation*}
\end{problem}

Remarkably, we can make connections between \(\mathsf{C}\Ri\mathsf{UOT}\) and another important family of problems involving concave functions of measures.

\begin{problem}[\textbf{\(\mathsf{U}\mathsf{O}\mathcal{Q}\mathsf{T}\)}]
        Suppose we are given two entropy functions \(\varphi_1, \varphi_2 : [0,+\infty) \to [0,+\infty]\) and a concave function \(\mathcal{Q} : \M^+(\X\times \Y)\to \R\). We define the \emph{\(\bphi\)-unbalanced optimal \(\mathcal{Q}\)-transport problem} as
    \begin{equation*}
     \mathsf{U}\mathsf{O}\mathcal{Q}\mathsf{T}(\alpha,\beta) \triangleq \inf_{\pi\in \M^{+}(\X\times \Y)} \mathcal{Q}(\pi)+ \D{\varphi_1}{\pi_1}{\alpha} 
     +\D{\varphi_2}{\pi_2}{\beta}.
    \end{equation*}
        \end{problem}
        
In the appendix we show how solving a $\mathsf{C}\Ri\mathsf{UOT}$ problem is the same as solving an $\mathsf{U}\mathsf{O}\mathcal{Q}\mathsf{T}$ problem from a certain concave functional $\mathcal{Q}$ built from $\Ri$. We say that two minimization problems are equivalent, or that one is an instance of the other, when they have the same minimizers in $\M^{+}(\X\times \Y)$.

\textbf{Entropic regularization.} In practice, a preferred way to solve UOT problems is using entropic regularization \citep{Chizat2018}. 
We consider adding such regularization also to $\mathsf{C}\Ri\textsf{UOT}$ problems.

\begin{problem}
Suppose we are given \(\varepsilon>0\), two entropy functions \(\varphi_1, \varphi_2 : [0,+\infty) \to [0,+\infty]\) and a convex function \(\Ri : \Ci(\X\times \Y) \to \R\cup\{+\infty\}\). We define the \emph{\(\varepsilon\)-entropic \(\Ri\)-regularized \(\bphi\)-unbalanced optimal transport problem} as
\begin{equation*}
    \begin{aligned}
         \mathsf{C}\Ri\textsf{UOT}_{ \varepsilon}(\alpha,\beta) \triangleq & \inf_{\pi, c} \int_{\X\times \Y} c(x,y)\d \pi(x,y)
         + \D{\varphi_1}{\pi_1}{\alpha} + \D{\varphi_2}{\pi_2}{\beta}
         + \Ri(c) + \varepsilon \D{\mathrm{KL}}{\pi}{\alpha \otimes \beta}.
\end{aligned}
\end{equation*}
\end{problem}

\subsection{Existence of Minimizers}
Our first result consists in establishing the existence of optimal solutions for the CRUOT problems. We will need the following definition.

\begin{defn}[Cost-Parametrized Regularizers]\label{cost-param-reg}
A convex function $\Ri : \Ci(\X\times \Y)\to [0,+\infty)$ is called \emph{cost-parametrized regularizer} if there exist $\F$ a compact subset of $\R^{d}$ and a family of costs $\{c_\theta\}_{\theta \in \F} \subset \mathcal{C}(\X\times \Y)$ s.t.
 \[\Ri(c) = \begin{cases}
        \tilde{\Ri}(\theta) & \text{if \(c=c_{\theta}\) for some \(\theta\in \F\)}\\
        +\infty & \text{otherwise,}
        \end{cases}\]
with $\tilde{\Ri} : \F \to [0,+\infty]$ a lower semicontinuous, coercive, convex function.
\end{defn}
For this family of cost-parametrized regularizers we show that we can find an optimal solution for $\cosre$ problems.
\begin{theorem}[Existence]\label{thm:exist}
Let $(\varphi_1, \varphi_2)$ be a pair of superlinear entropy functions satisfying \eqref{comp_cond} and $\varepsilon \geq 0$. Assume a cost-parametrized regularizer $\Ri$ as defined in Definition \ref{cost-param-reg} with $\{c_{\theta}\}_{\theta\in \F}$ a uniformly bounded from below family of continuous costs s.t. $c_{\theta_k}\to c_{\theta}$ uniformly whenever $\theta_k\to \theta$. Then the problem $\mathsf{C\Ri UOT}_{\varepsilon}(\alpha,\beta)$ admit at least one minimizer in $\F \times \M^{+}(\X\times \Y)$.
\end{theorem}


\subsection{Convergence of entropic minimizers}
 
 Remarkably, the next result guarantees that, when the cost-regularization involves a sufficiently regular cost-parametrized regularizer (Definition \ref{cost-param-reg}), the entropy-regularized $\mathsf{C}\Ri\textsf{UOT}_{ \varepsilon}$ problem actually converges to the original $\mathsf{C}\Ri\textsf{UOT}$ when $\varepsilon\to 0$. 

\begin{theorem}\label{thm:conv-entr-minimizers}
Let $\varepsilon_n\to 0$ and suppose that the assumptions of Theorem \ref{thm:exist} to hold with $\varphi_1,\varphi_2$ superlinear strictly convex satisfying \eqref{comp_cond_strong} or $\varphi_1 = \varphi_2 = \iota_{\{1\}}$ satisfying \eqref{comp_cond}. 
for every $\pi\in \M^+(\X\times \Y)$. Then the following hold.
\begin{enumerate}
    \item \(\mathsf{C\Ri UOT}_{\varepsilon_n}(\alpha,\beta) \stackrel{n\to +\infty}{\longrightarrow}\mathsf{C\Ri UOT}(\alpha,\beta)\).
   
   \item Consider a sequence \((\theta_*^{\varepsilon_n},\pi_*^{\varepsilon_n})_{n\in\N}\subset \F\times \M^+(\X\times \Y)\) s.t. \((\theta_*^{\varepsilon_n},\pi_*^{\varepsilon_n})\) minimizes \(\mathsf{C\Ri UOT}_{\varepsilon_n}(\alpha,\beta)\) for every \(n\in\N\). There exists a subsequence \((\theta_*^{\varepsilon_{n_k}},\pi_*^{\varepsilon_{n_k}})_{k\in\N}\) s.t.
    \[\theta_*^{\varepsilon_{n_k}}\to \theta_*, \qquad \pi_*^{\varepsilon_{n_k}} \wto \pi_*,\] where \((\theta_*,\pi_*)\) is optimal for \(\mathsf{C\Ri UOT}(\alpha,\beta)\).

\end{enumerate}
\end{theorem}

\subsection{IP-Cost-Regularized \textsf{UOT}}

Fix \(\X\subset \R^p\) and \(\Y\subset \R^q\) two compact domains of the respective euclidean space. 

\paragraph{The Gromov-Wasserstein problem.}
Let us begin by recalling the definition of \textsf{GW} problem \citep{sturm2023space, memoli2011gromov, chowdhury2019gromov}.

\begin{problem}[\textsf{GW}]
        Fix \(p\in [1,\infty)\). Consider two continuous cost functions \(c_{\X} : \X\times \X \to \R\) and \(c_{\Y} : \Y \times \Y\to \R\).
        The $\mathsf{GW}_p$ problem is defined as
        \begin{equation*}
           \inf_{\pi\in \Pi(\alpha,\beta)} \left[\int_{(\X\times \Y)^2} \| c_{\X}(x,x') -
        c_{\Y}(y,y')\|^p \d (\pi\otimes \pi)\right]^{\frac{1}{p}}.
         \end{equation*}
\end{problem}
We are particularly interested in the following particular case.
\begin{problem}[\textsf{GW-IP}]
        Let \(c_{\X}(x,x') = -\< x,x'\>\) and \(c_{\Y}(y,y') = -\< y,y'\>\), we denote by $\mathsf{GW\text{-}IP}(\alpha,\beta)$ the problem
\[     
    \inf_{\pi\in \Pi(\alpha,\beta)} \int_{(\X\times \Y)^2} \| \<x,x'\> - \<y,y'\>\|^2 \d (\pi\otimes \pi).
\]
\end{problem}
The main reason for our interest is the following result \citep{vayer2020contribution} connecting \textsf{GW-IP} and cost-regularized optimal transport problems.

\begin{prop}\label{prop_gwip}
    Let $r>0$. Denote 
    $\F_r := \{M\in \R^{q\times p} \, \| \, \Vert M\Vert_F\leq r\}$. Then, \textsf{\textup{GW-IP}} and the problem 
    $$\min_{\pi \in \Pi(\alpha,\beta)} \min_{M\in \F_r} -\int_{\X\times \Y} \< Mx,y\> \d \pi(x,y)$$
    are equivalent, i.e. they have the same minimizers in $\pi$.
\end{prop}

Inspired by the previous proposition, we define the following class of cost-regularized unbalanced optimal transport problems using inner product costs parametrized by linear transformations.

\begin{problem}[\(\mathsf{C\mathcal{R}_{r}UOT}\)]\label{problem_RipUOT}
        Suppose we are given \(r>0\) and \(\varepsilon\geq 0\). Consider the family of matrices \(\mathcal{F}_r = \left\{ M\in \R^{q\times p} \, \| \, \Vert M\Vert_F \leq r \right\}\). We define the following problem
        
        \begin{equation*}
            \begin{aligned}
        \mathsf{C}\mathsf{\Ri_{\mathit{r}}UOT}_{\varepsilon}(\alpha,\beta) = \inf_{\substack{\pi\in \M^+(\X\times \Y) \\ M\in \F_r}} &-\int_{\X\times \Y} \< Mx,y\> \d\pi 
        + \D{\varphi_1}{\pi_1}{\alpha} \\
        &+ \D{\varphi_2}{\pi_2}{\beta}
        + \varepsilon \D{\mathrm{KL}}{\pi}{\alpha\otimes \beta}.
        \end{aligned}
        \end{equation*}
        When \(\varepsilon=0\) we will use the alternative notation \(\mathsf{C\Ri_{\mathit{r}}UOT}(\alpha,\beta)\).
\end{problem}

\begin{oss}
        Clearly the previous problem is an instance of the general cost-regularized optimal transport Problem \ref{ROT} with cost-parametrized regularizer \(\Ri_{r} : \Ci(\X\times \Y)\to [0,+\infty]\) defined by
        \[\Ri_{r}(c) = \begin{cases}
                0 & \text{if \(c(x,y) = - \< M x, y\>\) for \(M\in \F_r\)}\\
                +\infty & \text{otherwise}.
        \end{cases}\]
\end{oss}

Interestingly, it still has a connection with \textsf{GW-IP}.

\begin{theorem}\label{existence_RipOT}
        Suppose \(p\geq q\), $\varphi_1,\varphi_2$ superlinear strictly convex satisfying \eqref{comp_cond_strong} or $\varphi_1 = \varphi_2 = \iota_{\{1\}}$ satisfying \eqref{comp_cond}, fix \(r>0\) and \(\varepsilon\geq 0\).
        The problem \(\mathsf{C\Ri_{\mathit{r}}UOT}_{\varepsilon}(\alpha,\beta)\) admits minimizers \((M^*_{\varepsilon},\pi^*_{\varepsilon})\). Moreover, if \((M^*,\pi^*)\) minimizes \(\mathsf{C\Ri_{\mathit{r}}UOT}(\alpha,\beta)\), then 
        \(\pi^*\) minimizes \(\mathsf{GW\text{-}IP}(\pi^*_1,\pi^*_2)\).
\end{theorem}

\section{Entropic Maps}

We show that under mild regularity, optimal couplings are induced by deterministic maps.

\begin{defn}
        A map \(T: \X\to \Y\) is a \emph{Monge map} for the problem \(\mathsf{C\Ri_{\mathit{r}}UOT}(\alpha,\beta)\) if there exists \((M^*,\pi^*)\in \F_r\times \M^+(\X\times \Y)\) s.t. \((M^*,\pi^*)\) is optimal for
        \(\mathsf{C\Ri_{\mathit{r}}UOT}(\alpha,\beta)\) and \(\pi^* = (\mathrm{id},T)_{\#} \pi^*_1\).
\end{defn}

Denote $B_r := \{y\in \R^q \, \vert \, \Vert y\Vert \leq r \max_{x\in \X}\Vert x\Vert\}$. Observe that 
$M_{\#}\mu \in \M^+(B_r)$ for every $M\in \F_r$ and $\mu\in \M^+(\X)$.

\begin{theorem}\label{thm_monge_map_RipOT}
        Assume $\X\subset \R^p$, $\Y\subset \R^q$ compact subsets, \(p\geq q\), \(\varphi_1,\varphi_2\) superlinear strictly convex satisfying \eqref{comp_cond_strong} or $\varphi_1 = \varphi_2 = \iota_{\{1\}}$ satisfying \eqref{comp_cond} and \(\alpha\) absolutely continuous w.r.t. the Lebesgue measure on \(\X\). Then, there exists a Monge map for \(\mathsf{C\Ri_{\mathit{r}}UOT}_{\bphi}(\alpha,\beta)\).
         In particular, for every optimal couple \((M^*,\pi^*)\) for \(\mathsf{C\Ri_{\mathit{r}}UOT}(\alpha,\beta)\) there exists a map \(T_*\) s.t. \(\pi^* = (\mathrm{id},T_*)_{\#} \pi^*_1\). Moreover, if \(M^*\) is surjective then we can take
        \[T_* = -\nabla f_* \circ M^*\]
        with \(f_*\in \Ci(B_r)\) an optimal Kantorovich potential for the problem \(\mathsf{OT}^{c_{\mathrm{ip}}}(M^*_{\#}\pi^*_1, \pi^*_2)\) differentiable \(M^*_{\#}\pi^*_1\)-a.e.,

    where $c_{\mathrm{ip}} =  -\langle y', y\rangle$ for every $y, y' \in \R^{q}$.
\end{theorem}

Clarified the existence of a Monge map for the problem \(\mathsf{C\Ri_{\mathit{r}}UOT}_{\bphi}(\alpha,\beta)\) we turn to the task of its estimation using entropic regularization to leverage the computational advantages that it carries.
We denote $(M^{\varepsilon}, \pi^{\varepsilon})$ solutions to \(\mathsf{C\Ri_{\mathit{r}}UOT}_{\bphi, \varepsilon}(\alpha,\beta)\) and $\alpha^{\varepsilon}, \beta^{\varepsilon}$ the marginals of $\pi^{\varepsilon}$.
We fix two sequences \((\varepsilon_n)_{n\in\N},(\varepsilon_j')_{j\in\N}\subset (0,+\infty)\) s.t. \(\varepsilon_n,\varepsilon_j'\to 0\).

\begin{defn}
    For every \(j,n\in \N\) we define the \emph{entropic map} \(T_{j,n} : \R^p \to \R^q\) as follows
    \begin{equation}\label{eq_entr_map}
        T_{j,n}(x) = \frac{\int_{\Y} y \exp\left[ \frac{1}{\varepsilon_n} ( g_{j,n}(y) + \< M^{\varepsilon_j'}x,y\>) \right] \d \beta^{\varepsilon_j'}(y)}{\int_{\Y} \exp\left[ \frac{1}{\varepsilon_n} ( g_{j,n}(y) + \< M^{\varepsilon_j'}x,y\>) \right] \d \beta^{\varepsilon_j'}(y)},
    \end{equation}
    where \((f_{j,n},g_{j,n})\in \Ci(B_r)\times \Ci(\Y)\) are optimal for 
   \begin{equation*}
\begin{split}
&\sup_{f,g \in \mathcal{C}(B_r) \times \mathcal{C}(\mathcal{Y})} \;  
   \int_{B_r} f \, d M^{\varepsilon_j'}_{\#}\alpha^{\varepsilon_j'} 
 + \int_{\mathcal{Y}} g \, d\beta^{\varepsilon_{j}'} \\
& - \varepsilon_n \int_{B_r \times \mathcal{Y}} 
   \left[ 
     \exp\left(\frac{f \oplus g - c_{\mathrm{ip}}}{\varepsilon_n}\right) - 1
   \right] 
   d\big(M^{\varepsilon_{j}'}_{\#}\alpha^{\varepsilon_{j}'} \otimes \beta^{\varepsilon_{j}'}\big).
\end{split}
\end{equation*}

\end{defn}

\begin{theorem}\label{thm_entropic_map_conv}
        Assume $\X\subset \R^p$, $\Y\subset \R^q$ compact subsets, \(p\geq q\), \(\varphi_1,\varphi_2\) superlinear strictly convex satisfying \eqref{comp_cond_strong} or $\varphi_1 = \varphi_2 = \iota_{\{1\}}$ satisfying \eqref{comp_cond} and \(\alpha\) absolutely continuous w.r.t. the Lebesgue measure on \(\X\). Assume also $B_r\subset \Y$. Then there exists a subsequence $(\varepsilon_{j_h}')_{h\in\N}$ s.t. $M^{\varepsilon_{j_h}'}\to M^*$ optimal for \(\mathsf{C\Ri_{\mathit{r}}UOT}_{\bphi}(\alpha,\beta)\). 
        Moreover, suppose \(M^*\) surjective and that \( M^*_{\#}\alpha^*\) and \(\beta^*\) satisfy the Assumptions A1-A3 in \citep{pooladian2021entropic} for $\alpha\geq 2$. Then
        \[\lim_{n\to +\infty} \lim_{h\to +\infty} T_{j_h,n} = T_* \quad \text{in \(L^2(\alpha^*)\),}\]
        where \(T_*\) is a Monge map for \(\mathsf{C\Ri_{\mathit{r}}UOT}(\alpha,\beta)\) which pushes \(\alpha^*\) to \(\beta^*\).
\end{theorem}

\section{Block coordinate descent algorithm for \(\mathsf{C\Ri_{\mathit{r}}UOT}_{\varepsilon}\)}

Consider the setup of the previous section. To approximate a solution for Problem \ref{problem_RipUOT} we use the following block coordinate descent algorithm \citep{bertsekas1997nonlinear, tseng2001convergence}:

\begin{align*}
\pi^{k+1} &= \argmin_{\pi\in \M^+(\X\times \Y)}
   - \int_{\mathcal{X} \times \mathcal{Y}} \langle M_k x, y \rangle \, d\pi(x,y) \notag 
+ \mathrm{D}_{\varphi_1}(\pi_1 \mid \alpha) + \mathrm{D}_{\varphi_2}(\pi_2 \mid \beta) 
   + \varepsilon \mathrm{D}_{\mathrm{KL}}(\pi \mid \alpha \otimes \beta), \\
M_{k+1} &= \argmin_{M\in \F_r}
   - \int_{\mathcal{X} \times \mathcal{Y}} \langle Mx, y \rangle \, d\pi^{k+1}(x,y).
\end{align*}

It is a well-studied problem in the optimization and machine learning literature. Adapting Lemma 4.2.2 in \citep{vayer2020contribution} we get

\begin{lemma}\label{lemma_gwip}
        Fix \(r\in (0,+\infty)\) and \(\pi\in \M^+(\X\times \Y)\). Then 
        \[r \left\Vert \int_{\X\times \Y} y x^T \d \pi(x,y) \right\Vert_F = \sup_{M\in \F_r} \int_{\X\times \Y} \<M x, y\> \d \pi(x,y)\]
        and supremum is achieved by
        \[M(\pi) = \begin{cases}
        \frac{r}{\left\Vert \int_{\X\times \Y} y x^T \d \pi(x,y)\right\Vert_F} \int_{\X\times \Y} y x^T \d \pi(x,y) & \text{if $\int_{\X\times \Y} y x^T \d \pi(x,y)\neq 0$} \\ 0 & \text{otherwise.} \end{cases}\]
\end{lemma}

By Lemma \ref{lemma_gwip} the block coordinate descent algorithm becomes

\begin{mdframed}[linewidth=0.5pt, roundcorner=4pt, innertopmargin=5pt, innerbottommargin=5pt, innerleftmargin=8pt, innerrightmargin=8pt, userdefinedwidth=\linewidth]
\begin{align}\label{eq_BCD}
   \pi^{k+1} &= \argmin_{\pi\in \M^+(\X\times \Y)} 
   -\int_{\X\times \Y} \< M_{k} x, y\> \, d\pi(x,y)\notag + \D{\varphi_1}{\pi_1}{\alpha} 
   + \D{\varphi_2}{\pi_2}{\beta} 
   + \varepsilon \D{\mathrm{KL}}{\pi}{\alpha\otimes \beta}\notag\\
   M_{k+1} &=  A_{r}(\pi^{k+1}) \int_{\X\times \Y} y x^T \, d \pi^{k+1}(x,y),
\end{align}
\end{mdframed}

 where $A_{r}(\pi^{k+1}) = \frac{r}{\left\Vert \int_{\X\times \Y} y x^T \d \pi^{k+1}(x,y)\right\Vert_{F}} $.

We now give a simple convergence result for the above algorithm in the practical case of discrete measures.

\begin{theorem}\label{thm:conv-algo}
        Let \(\X = \{x_i\}_{i=1}^n\subset \R^p\), \(\Y = \{y_j\}_{j=1}^m\subset \R^q\), \(\alpha=\sum_{i=1}^n a_i \delta_{x_i}\), \(\beta=\sum_{j=1}^m b_j \delta_{y_j}\) and \(\varepsilon,r>0\). Suppose \(\{a_i\}_{i=1}^n, \{b_j\}_{j=1}^m\subset (0,+\infty)\) and the entropy functions \(\varphi_1\) and \(\varphi_2\) to be superlinear.
        Then, any limit point of the sequence \(((M_k,\pi^k))_{k\in\N}\) defined by the block coordinate descent scheme \eqref{eq_BCD} is a stationary point of the objective function of \(\mathsf{C\mathcal{R}_r UOT}_{\varepsilon}(\alpha,\beta)\).
\end{theorem}




\section{Applications to single-cell multiomics alignments}

We evaluate the effectiveness of our Algorithm \ref{eq_BCD} on two \textbf{single-cell multi-omics datasets} \citep{argelaguet2021statistical}. Each dataset consists of two tables that record different cellular characteristics (modalities), measured on cells of distinct types. The two modalities live in Euclidean spaces of different dimensions, and our goal is to align the cells across modalities with respect to their type.  

Formally, we assign uniform probability measures $\alpha$ and $\beta$ to the source and target datasets, respectively. For each choice of $\varepsilon,\varepsilon',\lambda > 0$, we compute an entropic map $T_{\varepsilon,\varepsilon'}$ (see \eqref{eq_entr_map}) that approximates a Monge map for the cost-regularized problem $\Ri_r\mathsf{UOT}_{\bphi}(\alpha,\beta)$ (see Theorem~\ref{thm_entropic_map_conv}). Since alignments are always computed from the higher- to the lower-dimensional modality, we denote by \emph{source modality data} the table containing the higher-dimensional measurements, and by \emph{target modality data} the table containing the lower-dimensional ones. Importantly, in these datasets, the two modalities admit a one-to-one correspondence: each source measurement has a unique paired target measurement from the same cell, and every target cell appears in the source data.  

To evaluate performance in more challenging conditions, we additionally simulate \textbf{unbalancedness} by subsampling the source and target data with cell-type-dependent proportions, thereby breaking the one-to-one correspondence.  

All experiments are carried out in a \textbf{supervised setting}, where the cell type (label) is available for both modalities. Performance is quantified using \textbf{Label Transfer Accuracy (LTA)} \citep{demetci2022scot, demetci2022scotv2}, defined as the accuracy of predictions on aligned source data (in the target space) obtained by a $k$-nearest neighbors classifier trained on the target modality.  

In all experiments, we set the entropy functions to $\varphi_1 = \varphi_2 = \lambda \varphi_{\mathrm{KL}}$, where $\varphi_{\mathrm{KL}}(x) = x \log(x) - x + 1$. We fix $k=5$ for the $k$-NN classifier used in computing LTA, and set $\varepsilon = 5 \times 10^{-3}$ and $r = 1$ for the constraint set $\F_r$. The remaining hyperparameters $\varepsilon'$ and $\lambda$ are tuned by grid search.  

\subsection{scGEM dataset}

The first dataset we use is the scGEM dataset \citep{cheow2016single}, \citep{demetci2022scot, demetci2022scotv2} containing the gene expression and DNS methylation modalities of 177 human somatic cells. The source modality is the gene expression one, which have dimension \(p=34\), while the DNA methylation is the target modality, of dimension \(q=27\).
The task is to match source and target datasets using an entropic map from the source to the target. In the left column of Table of Figure~\ref{fig:myimage} are described the results of \(C\Ri_r\mathsf{UOT}\) on the full scGEM dataset when varying the parameter \(\lambda\) and the same kind of results is reported in the right column of Table of Figure~\ref{fig:myimage} for the randomly subsampled scGEM dataset. 
In particular, for the latter experiment, we randomly pick two cell types and subsampled at \(30\%\) the cells of the first type in the gene expression domain (source) and at \(30\%\) the cells of the second type in the DNA methylation domain (target). 



\begin{figure}[htbp]
    \centering

    \begin{subfigure}{0.45\textwidth}
        \centering
        \includegraphics[width=\linewidth]{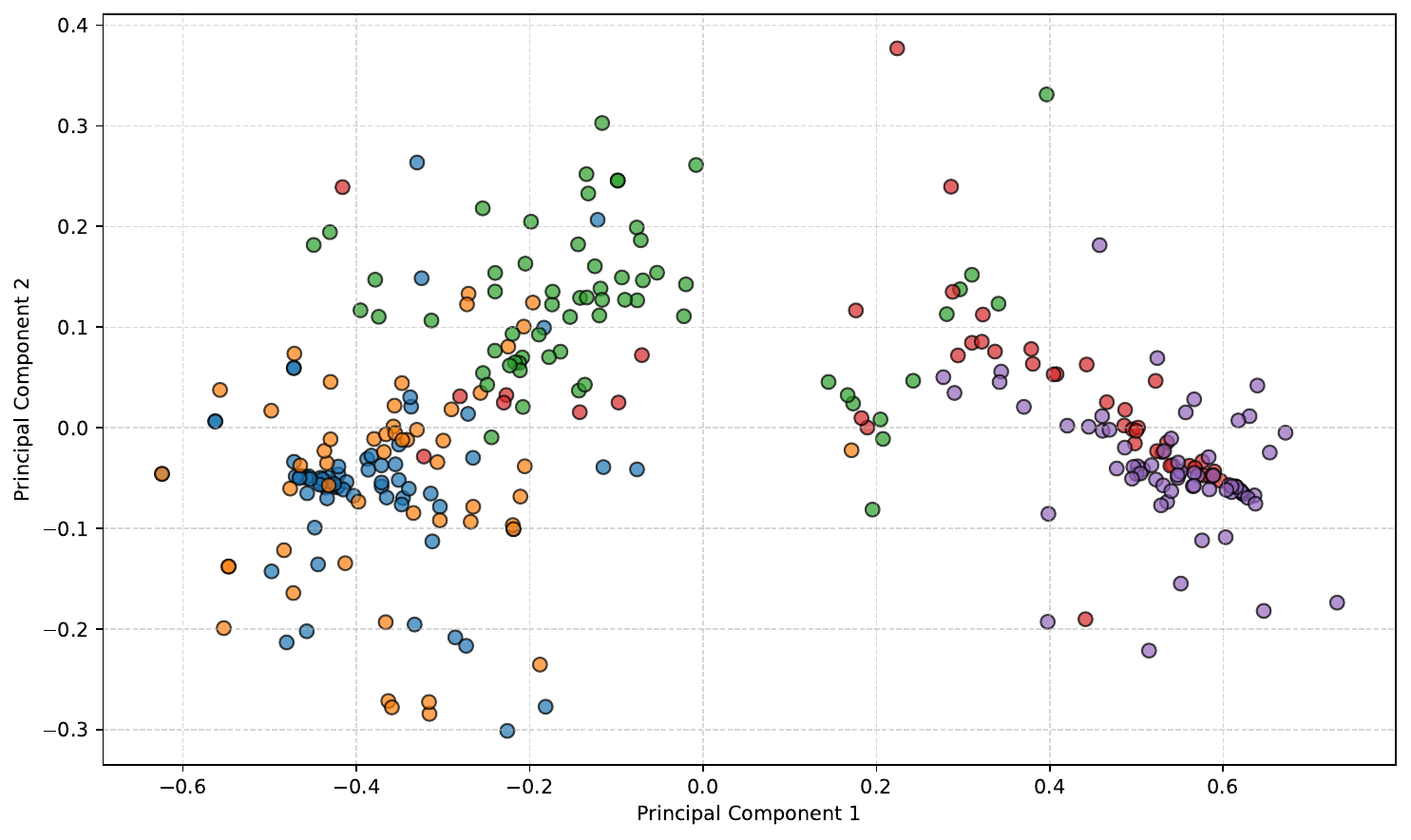}
        \caption{Entropic map alignment of the subsampled scGM dataset with $\lambda=1.0$ using two-dimensional PCA. Different colours refer to different cell types.}
        \label{fig:alignment_sub_scGM}
    \end{subfigure}
    \hfill
    \begin{subfigure}{0.45\textwidth}
        \centering
        \includegraphics[width=\linewidth]{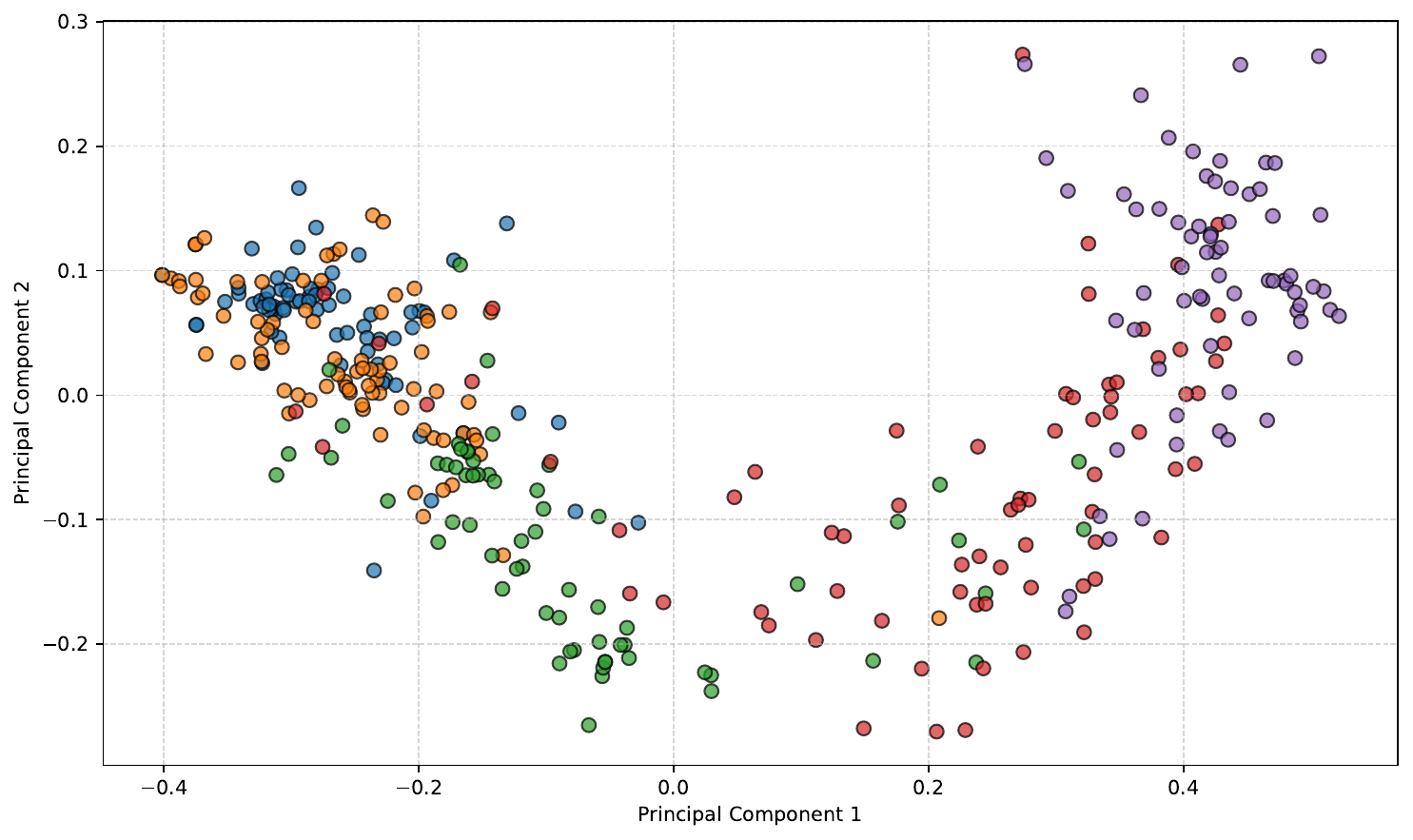}
        \caption{Entropic map alignment of the full scGM dataset with $\lambda=1.3$ using two-dimensional PCA. Different colours refer to different cell types.}
        \label{fig:alignment_scGM}
    \end{subfigure}

    \caption{Visualization of entropic map alignments for subsampled and full scGM datasets.}
    \label{fig:alignment_scGM_both}
\end{figure}


\begin{figure}[htbp]
    \centering

    \begin{minipage}{0.48\textwidth}
        \centering
        \includegraphics[width=\textwidth]{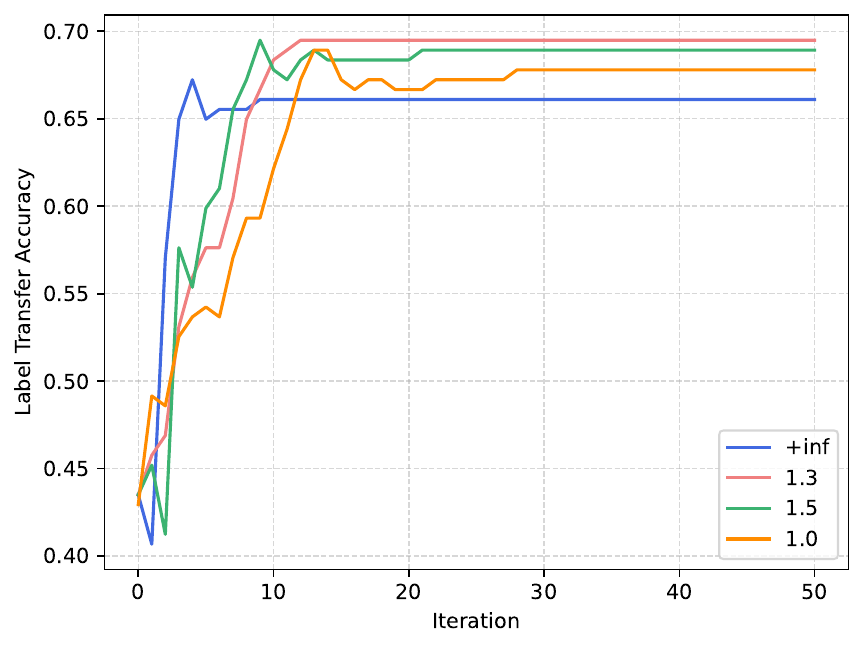}
     \end{minipage}
    \hfill
    \begin{minipage}{0.48\textwidth}
        \centering
        \begin{tabular}{cc !{\vrule width 1.2pt} cc}
            \toprule
            \multicolumn{2}{c!{\vrule width 1.2pt}}{Subsampled scGEM} &
            \multicolumn{2}{c}{Full scGEM} \\
            \cmidrule(r){1-2} \cmidrule(l){3-4}
            $\lambda$ & LTA & $\lambda$ & LTA \\
            \midrule
            $+\infty$ & 0.477 & $+\infty$ & 0.661 \\
            1.5       & 0.638 & 2.0       & 0.667 \\
            1.3       & 0.665 & 1.5       & 0.672 \\
            1.0       & \textbf{0.664} & 1.3 & \textbf{0.689} \\
            0.7       & 0.604 & 1.0       & 0.661 \\
            \bottomrule
        \end{tabular}
    \end{minipage}

    \caption{Plots of the LTA of the alignments for the full scGEM dataset
    obtained using the entropic map associated to the couple \((M,P)\) at each 
    iteration of Algorithm~\ref{eq_BCD}, and the corresponding LTA accuracies 
    for subsampled and full scGEM.}\label{fig:myimage}
\end{figure}

Note that the case \(\lambda=+\infty\) corresponds to the problem \(C\mathcal{R}_r\mathsf{UOT}_{\varepsilon}(\alpha,\beta)\) with $\varphi_1=\varphi_2=\iota_{\{1\}}$, which can be seen as an entropic-regularized version of the Gromov-Wasserstein problem \(\mathsf{GW\text{-}IP}(\alpha,\beta)\) (see Proposition~\ref{prop_gwip}).

We observe that the unbalanced alignments (\(\lambda< +\infty\)), consistently outperform those obtained with the balanced formulation (\(\lambda=+\infty\), \citep{sebbouh2024structured}) . In particular, on the subsampled scGEM dataset, introducing unbalancedness yields a substantial improvement: the method effectively compensates for differences in cell-type proportions and the lack of one-to-one correspondence caused by subsampling. This highlights the importance of relaxing the mass conservation constraint in scenarios where the datasets exhibit sampling biases or partial overlap.

\subsection{SNAREseq dataset}
The second dataset we use is the SNAREseq dataset \citep{chen2019high}, \citep{demetci2022scot, demetci2022scotv2}, containing the chromatine accessibility (ATAC-seq) and gene expression (RNA-seq) of 1047 single cells of 4 different types. The source ATAC-seq modality has dimension \(p=19\), while the target RNA-seq modality has dimension \(q=10\). As in the previous experiment, we align the source and target modality datasets using an entropic map from the source to the target. Results on both the full dataset and an unbalanced subsampling show that our method remains robust when cell-type proportions differ, with complete results and implementation details provided in the Appendix.
\section{Conclusion and future work}
We introduced a unified framework that combines unbalanced optimal transport (UOT) with cost learning, supported by theoretical guarantees on the existence of minimizers, convergence, and Monge maps. On the computational side, we highlighted in the appendix the potential of low-rank parametrization of transport plans and regularizers which forcing sparcity in line with recent advances on scalable OT methods \citep{sebbouh2024structured}. A systematic treatment of these approximations within our framework could substantially reduce memory and runtime complexity while preserving theoretical guarantees.

\section*{Acknowledgments} KP acknowledges the partial support of the project PNRR - M4C2 - Investimento 1.3, Partenariato Esteso PE00000013 - ``FAIR - Future Artificial Intelligence Research'' - Spoke 1 ``Human-centered AI'', funded by the European Commission under the NextGeneration EU programme. KP would like to thank Dario Trevisan and Andrea Agazzi for helpful discussions. KP was supported by Max Planck Institute for Mathematics in the Sciences in Leipzig.

\appendix
\renewcommand\thesection{\Alph{section}}


\section{Appendix for Section 3}
\label{app:sec3}

\paragraph{Notation.}
We write $\M^+(Z)$ for finite nonnegative Radon measures on a compact metric space $Z$, and $\Ci(Z)$ for continuous real-valued functions on $Z$. For $\pi\in \M^+(\X\times\Y)$, denote marginals by $\pi_i = p_i{}_\#\pi$, $i=1,2$. For entropy functions $\varphi_1,\varphi_2$, we use the $\varphi$-divergences $\D{\varphi_i}{\cdot}{\cdot}$, and $\D{\mathrm{KL}}{\cdot}{\cdot}$ is the Kullback-Leibler divergence. We keep $\rho=\alpha\otimes\beta$ in entropic terms, where $\alpha$, $\beta$ is the source and target measure respectively.

\subsection{Cost Regularized Unbalanced Optimal Transport}
\label{app:subsec-ruot}

\begin{problem}[\(\mathsf{C}\Ri\mathsf{UOT}\) problems]\label{app:prob-RUOT}
Let $\Ri : \Ci(\X\times\Y)\to\R \cup \{+\infty\}$ be a convex function. We define the following cost-regularised problem
\[
\mathsf{C}\Ri\textsf{UOT}(\alpha,\beta)
\triangleq \inf_{\pi\in \M^+(\X\times\Y)} \ \inf_{c\in \Ci(\X\times\Y)}
\Big\{\ \textstyle\int_{\X\times\Y} c\,\mathrm{d}\pi + \D{\varphi_1}{\pi_1}{\alpha} + \D{\varphi_2}{\pi_2}{\beta} + \Ri(c)\ \Big\}.
\]
For $\varepsilon>0$, we define the entropic cost-regularised unbalanced optimal transport
\[
\textsf{C}\Ri\textsf{UOT}_{\varepsilon}(\alpha,\beta)
\triangleq \inf_{\pi,c}\Big\{\ \textstyle\int c\,\mathrm{d}\pi + \D{\varphi_1}{\pi_1}{\alpha} + \D{\varphi_2}{\pi_2}{\beta} + \Ri(c) + \varepsilon\,\D{\mathrm{KL}}{\pi}{\rho}\ \Big\}.
\]
In the balanced case $\bphi=(\iota_{(=)},\iota_{(=)})$ we write $\textsf{C}\Ri\textsf{OT}_\varepsilon(\alpha,\beta)$ and $\textsf{C}\Ri\textsf{OT}(\alpha,\beta)$.
\end{problem}

For completeness we recall the compatibility conditions
\begin{equation}\label{comp_cond}
    \left( m(\alpha) \mathrm{dom}(\varphi_1) \right) \cap \left( m(\beta) \mathrm{dom}(\varphi_2) \right)\neq \emptyset
\end{equation}
and the stronger one 
\begin{equation}\label{comp_cond_strong}
    \begin{split}
    &\left[ \mathrm{Int} (m(\alpha) \mathrm{dom}(\varphi_1)) \cap \left( m(\beta) \mathrm{dom}(\varphi_2) \right) \right]\\
    &\quad \cup \left[ \left( m(\alpha) \mathrm{dom}(\varphi_1) \right) \cap \Int ( m(\beta) \mathrm{dom}(\varphi_2) ) \right] \neq \emptyset.
    \end{split}
\end{equation}


\begin{oss}[Feasibility conditions]
\leavevmode\\
\textbf{Balanced case.} 
If $\varphi = (\iota_{=}, \iota_{=})$, feasibility requires the compatibility condition \textnormal{(1)}. 
Indeed, if the supports of $\alpha$ and $\beta$ do not overlap under the marginal maps $m(\cdot)$, 
then there exists no admissible coupling with the prescribed marginals, and consequently 
$\Pi(\alpha, \beta) = \emptyset$, making $\mathrm{CROT}(\alpha, \beta) = +\infty$. 
Condition~\textnormal{(1)} thus ensures that at least some mass from $\alpha$ can be transported to $\beta$ 
without violating the marginal constraints. 

\medskip
\textbf{Unbalanced case.} 
When general entropy functions $\varphi_1, \varphi_2$ are used, the feasibility of the relaxed formulation 
requires either \textnormal{(1)} or the stronger condition~\textnormal{(2)}. 
The latter guarantees that the effective domains of the divergences are compatible, 
so that partial mass transfer is possible even when the supports of $\alpha$ and $\beta$ 
do not perfectly coincide. 
In practice,~\textnormal{(2)} prevents degenerate situations where both divergences 
assign infinite cost to any nontrivial measure, ensuring that the unbalanced OT functional 
admits at least one finite value.
\end{oss}


\subsection{From convex to concave functionals on plans}

\begin{problem}[\(\mathsf{UO}\mathcal{Q}\mathsf{T}\) problems]\label{app:prob-OQT}
Given entropy functions $\varphi_1,\varphi_2$ and a concave $\mathcal{Q}:\M^+(\X\times\Y)\to\R$, define
\[
\mathsf{UO}\mathcal{Q}\mathsf{T}(\alpha,\beta)
\triangleq \inf_{\pi\in \M^+(\X\times\Y)} \ \mathcal{Q}(\pi)+\D{\varphi_1}{\pi_1}{\alpha}+\D{\varphi_2}{\pi_2}{\beta},
\]
and $\mathsf{UO}\mathcal{Q}\mathsf{T}_{\varepsilon}$ by adding $\varepsilon\,\D{\mathrm{KL}}{\pi}{\rho}$.
\end{problem}

For fixed $\varepsilon\ge 0$, set
\[
\mathcal{J}_{\bphi,\Ri}(\pi):=\inf_{c\in\Ci(\X\times\Y)} \int c\,\mathrm{d}\pi+\D{\varphi_1}{\pi_1}{\alpha}+\D{\varphi_2}{\pi_2}{\beta}+\Ri(c)+\varepsilon\,\D{\mathrm{KL}}{\pi}{\rho},
\]
\[
\mathcal{I}_{\bphi,\mathcal{Q}}(\pi):=\mathcal{Q}(\pi)+\D{\varphi_1}{\pi_1}{\alpha}+\D{\varphi_2}{\pi_2}{\beta}+\varepsilon\,\D{\mathrm{KL}}{\pi}{\rho}.
\]

\begin{prop}\label{app:prop:OQT-from-R}
\textbf{(From $\Ri$ to $\mathcal{Q}$).}

Fix $\varepsilon\ge 0$. Let $\varphi_1,\varphi_2:[0,+\infty)\to[0,+\infty]$ be entropy functions and
$\Ri:\Ci(\X\times\Y)\to[0,+\infty]$ be convex such that, for every $\pi\in \M^+(\X\times\Y)$,
\[
\inf_{c\in \Ci(\X\times\Y)}\ \int_{\X\times\Y} c\,\mathrm{d}\pi+\Ri(c)\ \in \ \R.
\]
Define the concave functional $\mathcal{Q}_{\Ri}:\M^+(\X\times\Y)\to\R$ by
\[
\mathcal{Q}_{\Ri}(\pi)\ \triangleq \ \inf_{c\in \Ci(\X\times\Y)}\ \int_{\X\times\Y} c\,\mathrm{d}\pi+\Ri(c).
\]
Then, for all $\pi\in \M^+(\X\times\Y)$,
\[
\mathcal{J}_{\bphi,\Ri}(\pi)\ =\ \mathcal{I}_{\bphi,\mathcal{Q}_{\Ri}}(\pi),
\]
and in particular the minimizers coincide and
\[
\mathsf{C}\Ri\textsf{UOT}_{\varepsilon}(\alpha,\beta)\ =\ \mathsf{U}\mathsf{O}\mathcal{Q}_{\Ri}\textsf{T}_{\bphi,\varepsilon}(\alpha,\beta).
\]
\end{prop}

\begin{proof}
$\mathcal{Q}_{\Ri}$ is concave as an infimum of affine maps of $\pi$. The identity
$\mathcal{J}_{\bphi,\Ri}=\mathcal{I}_{\bphi,\mathcal{Q}_{\Ri}}$ follows immediately from the definitions,
hence the equality of values and minimizers.
\end{proof}

\begin{prop}\label{app:prop:R-from-OQT}
\textbf{(From $\mathcal{Q}$ to $\Ri$).}
Fix $\varepsilon\ge 0$. Let $\mathcal{Q}:\M^+(\X\times\Y)\to\R$ be concave and weakly upper semicontinuous.
Define $\bar{\mathcal{Q}}:\M(\X\times\Y)\to\R\cup\{-\infty\}$ by
\[
\bar{\mathcal{Q}}(\gamma)\ =\
\begin{cases}
\mathcal{Q}(\gamma), & \gamma\in \M^+(\X\times\Y),\\
-\infty, & \text{otherwise},
\end{cases}
\]
and the convex functional $\Ri_{\mathcal{Q}}:\Ci(\X\times\Y)\to\R\cup\{+\infty\}$ by
\[
\Ri_{\mathcal{Q}}(c)\ \triangleq \ (-\bar{\mathcal{Q}})^*(-c).
\]
Then, for all $\pi\in \M^+(\X\times\Y)$,
\[
\mathcal{J}_{\bphi,\Ri_{\mathcal{Q}}}(\pi)\ =\ \mathcal{I}_{\bphi,\mathcal{Q}}(\pi),
\]
and, in particular,
\[
\mathsf{UO}\mathcal{Q}\mathsf{T}_{\varepsilon}(\alpha,\beta)\ =\ 
\mathsf{C}\Ri_{\mathcal{Q}}\mathsf{UOT}_{\varepsilon}(\alpha,\beta),
\]
with the same set of minimizers.
\end{prop}

\begin{proof}
Since $-\bar{\mathcal{Q}}$ is proper, convex and weakly lower semicontinuous on $\M(\X\times\Y)$,
Fenchel–Moreau yields
\[
\bar{\mathcal{Q}}(\pi)
= -(-\bar{\mathcal{Q}})^{**}(\pi)
= -\sup_{c\in \Ci(\X\times\Y)}\Big\{\textstyle\int c\,\mathrm{d}\pi - (-\bar{\mathcal{Q}})^*(c)\Big\}
= \inf_{c\in \Ci(\X\times\Y)}\Big\{\textstyle\int c\,\mathrm{d}\pi + (-\bar{\mathcal{Q}})^*(-c)\Big\}.
\]
Adding $\D{\varphi_1}{\pi_1}{\alpha}+\D{\varphi_2}{\pi_2}{\beta}+\varepsilon\,\D{\mathrm{KL}}{\pi}{\rho}$ on both sides gives
\[
\mathcal{I}_{\bphi,\mathcal{Q}}(\pi)
= \inf_{c\in \Ci(\X\times\Y)}\Big\{\textstyle\int c\,\mathrm{d}\pi + \Ri_{\mathcal{Q}}(c)\Big\}
+ \D{\varphi_1}{\pi_1}{\alpha}+\D{\varphi_2}{\pi_2}{\beta}+\varepsilon\,\D{\mathrm{KL}}{\pi}{\rho}
= \mathcal{J}_{\bphi,\Ri_{\mathcal{Q}}}(\pi).
\]
\end{proof}

\begin{oss}
\label{app:rmk:convex-analysis}
The conjugate is taken with respect to the duality $\langle \gamma,c\rangle=\int c\,\mathrm{d}\gamma$
between $\M^{+}(\X\times\Y)$ and $\Ci(\X\times\Y)$. Weak topologies are the ones induced by this pairing.
\end{oss}

\subsection{Proof of Theorem \ref{thm:exist}: Existence of minimizers for \texorpdfstring{$\mathsf{C}\Ri\mathsf{UOT}$}{\textsf{C}RUOT}}
For completeness we recall the definition of the class of cost-parametrized regularizers and the statement of Theorem \ref{thm:exist}.
\begin{defn}[Cost-Parametrized Regularizers]\label{cost-param-reg}
A convex function $\Ri \triangleq \Ci(\X\times \Y)\to [0,+\infty]$ is called \emph{cost-parametrized regularizer} if there exist $\F$ a compact subset of $\R^{d}$ and a family of costs $\{c_\theta\}_{\theta \in \F} \subset \mathcal{C}(\X\times \Y)$ s.t.
 \[\Ri(c) = \begin{cases}
        \tilde{\Ri}(\theta) & \text{if \(c=c_{\theta}\) for some \(\theta\in \F\)}\\
        +\infty & \text{otherwise,}
        \end{cases}\]
with $\tilde{\Ri} : \F \to [0,+\infty)$ a lower semicontinuous function.
\end{defn}

\begin{theorem}[Existence]\label{app:existence-RUOT}
Let $(\varphi_1, \varphi_2)$ be a pair of superlinear entropy functions satisfying \eqref{comp_cond} and $\varepsilon \geq 0$. Assume a cost-parametrized regularizer $\Ri$ as defined in Definition \ref{cost-param-reg} with $\{c_{\theta}\}_{\theta\in \F}$ a uniformly bounded from below family of continuous costs s.t. $c_{\theta_k}\to c_{\theta}$ uniformly whenever $\theta_k\to \theta$. Then the problem $\mathsf{C\Ri UOT}_{\varepsilon}(\alpha,\beta)$ admit at least one minimizer in $\F \times \M^{+}(\X\times \Y)$.

\end{theorem}
\begin{proof}

Consider
\[
J(\theta,\pi)\;=\;\int_{\X\times\Y} c_\theta\,d\pi \;+\; \D{\varphi_1}{\pi_1}{\alpha}
\;+\; \D{\varphi_2}{\pi_2}{\beta} + \varepsilon\,\mathrm{KL}(\pi\,\|\,\rho)\;+\;\tilde{\mathcal R}(\theta),
\]
where $\rho$ is the reference measure in $\X\times \Y$ and $\rho = \alpha \times \beta$.
Using the facts that $c_\theta\ge L$ and the convexity of $\varphi_i$, one gets the standard mass–coercivity bound
\[
J(\theta,\pi)\geq m(\pi)\left(
L + \frac{m(\alpha)}{m(\pi)}\varphi_1\Big(\frac{m(\pi)}{m(\alpha)}\Big)
   + \frac{m(\beta)}{m(\pi)}\varphi_2 \Big(\frac{m(\pi)}{m(\beta)}\Big)
\right) 
,
\]
which tends to $+\infty$ uniformly in $\theta$ when $m(\pi)\to\infty$, since
$\varphi_i$ are superlinear. Hence, the minimizers lie in
\[
\mathcal A \; \triangleq  \; \F \times \mathcal B_R^+,
\qquad
\mathcal B_R^+ \triangleq \{\pi\in\M^+(\X\times\Y): m(\pi)\le R\},
\]
for some $R>0$. Note that $\mathcal B_R^+$ is weakly compact by weak closedness and Banach-Alaoglou theorem.

So
$\mathcal A$ is compact for the product topology $\tau\triangleq\tau_{\mathrm{eucl}}\times\tau_{\mathrm{weak}}$, where $\tau_{\mathrm{eucl}}$ is the Euclidean topology on $\mathcal{F}$ and $\tau_{\mathrm{weak}}$ is the weak topology in $\mathcal{M}^{+}(\X \times \Y).$

It remains to show $J$ is $\tau$–l.s.c. Take a sequence $(\theta_{k}, \pi^{k})_{k}\subset \F \times \M^{+} (\X \times \Y) $ such that $(\theta_k,\pi^k)\to (\theta,\pi)$ in $\tau$.
Then $m(\pi^k)$ is bounded, $\pi^k$ weakly converges to \(\pi\), and $\pi_i^k$ weakly converges to $\pi_i$.
By the uniform convergence $c_{\theta_k}\to c_\theta$ and the boundedness of the masses $m(\pi^k)$,
\begin{align*}
\liminf_{k\to\infty}\int c_{\theta_k}\,d\pi^k
&\ge \liminf_{k\to\infty}\left(
-\norm{c_{\theta_k}-c_\theta}_\infty\,m(\pi^k) + \int c_\theta\,d\pi^k\right)
= \int c_\theta\,d\pi.
\end{align*}
The mappings $\pi\mapsto \D{\varphi_i}{\pi_i}{\cdot}$ are weakly l.s.c. in the marginals.
Also, $\mathrm{KL}(\cdot\,\|\,\rho)$ is weakly l.s.c. on compact metric spaces.
Therefore $J$ is l.s.c. on $\mathcal A$, and by Weierstrass theorem there exists a minimizer of $J$ on $\mathcal A$, see \citep[Box 1.1]{santambrogio2015optimal}.
Hence, $\mathrm{CRUOT}_\varepsilon(\alpha,\beta)$ admits at least one minimizer in $\F \times \M^{+}(\X\times \Y)$.

\end{proof}

\subsection{Proof of Theorem \ref{thm:conv-entr-minimizers}: Convergence of Entropic Minimizers}
For completeness, we restate Theorem \ref{thm:conv-entr-minimizers}.
\begin{theorem}
Let $\varepsilon_n\to 0$ and suppose that the assumptions of Theorem \ref{thm:exist} hold with $\varphi_1,\varphi_2$ superlinear strictly convex satisfying \eqref{comp_cond_strong} or $\varphi_1 = \varphi_2 = \iota_{\{1\}}$ satisfying \eqref{comp_cond}. 
Then the following hold.
\begin{enumerate}
    \item \(\mathsf{C\Ri UOT}_{\varepsilon_n}(\alpha,\beta) \stackrel{n\to +\infty}{\longrightarrow}\mathsf{C\Ri UOT}(\alpha,\beta)\).
    \item Consider a sequence \((\theta_*^{\varepsilon_n},\pi_*^{\varepsilon_n})_{n\in\N}\subset \F\times \M^+(\X\times \Y)\) s.t. \((\theta_*^{\varepsilon_n},\pi_*^{\varepsilon_n})\) minimizes \(\mathsf{C\Ri UOT}_{\varepsilon_n}(\alpha,\beta)\) for every \(n\in\N\). There exists a subsequence \((\theta_*^{\varepsilon_{n_k}},\pi_*^{\varepsilon_{n_k}})_{k\in\N}\) s.t.
    \[\theta_*^{\varepsilon_{n_k}}\to \theta_*, \qquad \pi_*^{\varepsilon_{n_k}} \wto \pi_*,\] where \((\theta_*,\pi_*)\) is optimal for \(\mathsf{C\Ri UOT}(\alpha,\beta)\).
\end{enumerate}
\end{theorem}
In order to prove Theorem \ref{thm:conv-entr-minimizers} we need the following results. First Lemma gives us convergence of values. The last two Lemmas are well-known results in the literature of unbalanced optimal transport problems about optimal marginals \cite{liero2018optimal}.

\begin{lemma}\label{lemma:conv_RUOT}
Suppose the same assumptions as in Theorem \ref{thm:conv-entr-minimizers} hold. Let $(\varepsilon_n)_{n\in\N}\subset(0,+\infty)$ with $\varepsilon_n\to0$.
Assume that there exists a sequence $(\eta_j)_{j\in\N}\subset(0,+\infty)$ with $\eta_j\to0$ such that for every $j\in\N$
there exists $\pi^j\in\M^+(\X\times\Y)$, $\theta_j\in \F$ with $\D{\mathrm{KL}}{\pi^j}{\alpha\otimes\beta}<+\infty$ and
\[
    \int_{\X\times\Y} c_{\theta_j}\,\mathrm{d}\pi^j
+\D{\varphi_1}{(\pi^j)_1}{\alpha}
+\D{\varphi_2}{(\pi^j)_2}{\beta}
+\tilde{\mathcal R}(\theta_j)
\le \mathrm{C}\Ri \mathrm{UOT}(\alpha,\beta)+\eta_j.
\]
Then $\mathrm{C}\Ri\mathsf{UOT}_{\varepsilon_n}(\alpha,\beta)\to \mathrm{C}\Ri\mathsf{UOT}(\alpha,\beta)$ as $n\to\infty$.
\end{lemma}

\begin{proof}
For each $n$ we have
\begin{align*}
    \cosre
\le \cosree
\le &\int_{\X\times\Y} c_{\theta_j}\,\mathrm{d}\pi^j
+\D{\varphi_1}{(\pi^j)_1}{\alpha}
+\D{\varphi_2}{(\pi^j)_2}{\beta}\\
&+\tilde{\mathcal R}(\theta_j)
+\varepsilon_n\,\D{\mathrm{KL}}{\pi^j}{\alpha\otimes\beta}.
\end{align*}
Hence
\[
\cosre
\le \liminf_{n\to\infty}\cosree
\le \limsup_{n\to\infty}\cosree
\le \cosre+\eta_j.
\]
Letting $j\to\infty$ gives the claim.
\end{proof}

\begin{lemma}[Fenchel--Kantorovich duality and optimal marginals, \citep{liero2018optimal, eyring2023unbalancedness}]
\label{lemma:FKdual}
Let $\varphi_1,\varphi_2:[0,+\infty)\to[0,+\infty]$ be proper l.s.c.\ strictly convex entropy functions, and $c\in\Ci(\X\times\Y)$. 
Consider the unbalanced optimal transport problem
\[
\mathsf{UOT}^c(\alpha,\beta)
= \inf_{\pi\in \M^+(\X\times \Y)}\int c\,\mathrm{d}\pi
+ \mathrm{D}_{\varphi_1}(\pi_1|\alpha)
+ \mathrm{D}_{\varphi_2}(\pi_2|\beta).
\]
Then its Fenchel--Kantorovich dual reads
\[
\mathsf{D}^c(\alpha,\beta)
= \sup
\left\{
-\int \varphi_1^*(-f)\,\mathrm{d}\alpha
-\int \varphi_2^*(-g)\,\mathrm{d}\beta
\,\middle|\,
(f,g)\in\Ci(\X)\times \Ci(\Y), \, f(x)+g(y)\leq c(x,y)
\right\}.
\]
If $(\pi_*,f_*,g_*)$ are optimal for the primal and dual problems, 
then the optimal marginals satisfy
\[
\frac{\mathrm{d}\pi_{*,1}}{\mathrm{d}\alpha}
= (\varphi_1^*)'(-f_*),
\qquad
\frac{\mathrm{d}\pi_{*,2}}{\mathrm{d}\beta}
= (\varphi_2^*)'(-g_*).
\]
Equivalently,
\[
\alpha_* = (\varphi_1^*)'(-f_*)\,\alpha,
\qquad
\beta_*  = (\varphi_2^*)'(-g_*)\,\beta.
\]
\end{lemma}
 This lemma guarantees we can approximate $\pi_*$ by discrete (simple) plans with the same marginals $\alpha_*,\beta_*$, while preserving continuity of the cost term.
\begin{lemma}[Block approximation, \citep{nutz2021introduction}]\label{lemma:block}
Suppose $\X$ and $\Y$ are compact metric spaces, 
and let $\mu \in \M^+(\X)$, $\nu \in \M^+(\Y)$. 
Fix a plan $\pi \in \Pi(\mu,\nu)$. 
Then, for every $\delta>0$, there exists a plan 
$\pi^\delta \in \Pi(\mu,\nu)$ such that 
\[
\pi^\delta \ll \mu \otimes \nu,
\qquad 
\frac{\mathrm{d}\pi^\delta}{\mathrm{d}(\mu\otimes\nu)} 
\text{ is bounded,}
\qquad
\pi^\delta \rightharpoonup \pi
\ \text{as } \delta\to0.
\]
In particular, for any continuous cost $c\in\Ci(\X\times\Y)$,
\[
\int_{\X\times\Y} c\,\mathrm{d}\pi^\delta 
\;\longrightarrow\;
\int_{\X\times\Y} c\,\mathrm{d}\pi 
\qquad \text{as } \delta\to0.
\]
\end{lemma}

\paragraph{Proof of Theorem \ref{thm:conv-entr-minimizers}}
Let $(\theta_*,\pi_*)\in \F\times \M^+(\X\times\Y)$ be an optimal couple for the unregularized problem 
$\mathsf{C}\mathcal{R}\mathsf{UOT}(\alpha, \beta)$ and denote by 
$\alpha_*:=\pi_{*,1}$ and $\beta_*:=\pi_{*,2}$ the first and second marginals of $\pi_*$. If we are in the balanced case, then
$\alpha_*=\alpha$ and $\beta_*=\beta$. Otherwise, in the unbalanced setting, the optimal marginals 
$\alpha_*,\beta_*$ are reweighted versions of $\alpha,\beta$ determined 
by the optimal dual potentials $(f_*,g_*)$ of the problem. Indeed, by the Fenchel--Kantorovich duality for the unbalanced problem 
(cf.~Lemma~\ref{lemma:FKdual}), the optimal marginals of the 
unregularized plan $\pi_*$ satisfy
\[
\frac{\mathrm{d}\pi_{*,1}}{\mathrm{d}\alpha}
= (\varphi_1^*)'(-f_*),
\qquad
\frac{\mathrm{d}\pi_{*,2}}{\mathrm{d}\beta}
= (\varphi_2^*)'(-g_*),
\]
where $(f_*,g_*)$ are the optimal dual potentials. Thus, we can write 
$\alpha_* = (\varphi_1^*)'(-f_*)\,\alpha$ and 
$\beta_* = (\varphi_2^*)'(-g_*)\,\beta$.

Let us set $\sigma_1 := (\varphi_1^*)'(-f_*)$ and $\sigma_2 := (\varphi_2^*)'(-g_*)$, 
so that $\alpha_* = \sigma_1\alpha$ and $\beta_* = \sigma_2\beta$. These $\sigma_i$ are bounded positive densities (since $f_*,g_*$ are bounded).

By Lemma~\ref{lemma:block}, there exists a sequence of couplings
$(\pi_*^\delta)_{\delta>0}\subset\Pi(\alpha_*,\beta_*)$ 
such that $\pi_*^\delta\rightharpoonup\pi_*$. 

Hence, for every $\eta>0$, we can find a plan $\pi^\eta\in\Pi(\alpha_*,\beta_*)$ such that
\[
\int c_{\theta_*}\,\mathrm{d}\pi^\eta
+\D{\varphi_1}{\alpha_*}{\alpha}
+\D{\varphi_2}{\beta_*}{\beta}
+\tilde{\mathcal R}(\theta_*)
\le \cosre +\eta.
\]
This $\pi^\eta$ is an $\eta$–optimal coupling for the unregularized problem.

Next we verify that $\pi^\eta$ has finite Kullback–Leibler divergence with respect 
to $\alpha\otimes\beta$. Using the change–of–measure formula, 
\[
\frac{\mathrm{d}\pi^\eta}{\mathrm{d}(\alpha\otimes\beta)}
=\frac{\mathrm{d}\pi^\eta}{\mathrm{d}(\alpha_*\otimes\beta_*)}
  \frac{\mathrm{d}(\alpha_*\otimes\beta_*)}{\mathrm{d}(\alpha\otimes\beta)}
=\frac{\mathrm{d}\pi^\eta}{\mathrm{d}(\alpha_*\otimes\beta_*)}\,
  \sigma_1\sigma_2.
\]
This decomposition uses the Radon–Nikodym derivative.The first factor is the density of $\pi^\eta$ w.r.t. its own marginals. The second factor comes from the change of measures $\alpha_*\otimes\beta_* = \sigma_1\sigma_2(\alpha\otimes\beta)$.

Since $\sigma_1,\sigma_2$ are bounded, the product density above is bounded, 
so $\D{\mathrm{KL}}{\pi^\eta}{\alpha\otimes\beta}<+\infty$.This shows the approximating sequence satisfies the finite–KL condition required by Lemma~\ref{lemma:conv_RUOT}.

Applying Lemma~\ref{lemma:conv_RUOT} with this family $(\pi^\eta)_\eta$ and $\theta_*$
yields the convergence of values~(1).

Finally, for the convergence of minimizers (point~(2)), 
the sequence $(\pi_*^{\varepsilon_n})_n$ is tight, 
since the coercivity estimate in Theorem~\ref{app:existence-RUOT} 
implies a uniform bound $m(\pi_*^{\varepsilon_n})\le R$. 
Thus, up to a subsequence, $\pi_*^{\varepsilon_n}\rightharpoonup\bar{\pi}$ 
weakly in $\M^+(\X\times\Y)$. The coercivity of the functional gives uniform mass bounds. Moreover, since $\F$ is compact we can also suppose $\theta_*^{\varepsilon_n}\to \bar{\theta}\in \F$.

By the uniform convergence of the costs $c_{\theta}$ 
and weak lower semicontinuity of the divergences, we have
\[
\int c_{\bar{\theta}}\,\mathrm{d}\bar{\pi}
+\D{\varphi_1}{\bar{\pi}_1}{\alpha}
+\D{\varphi_2}{\bar{\pi}_2}{\beta}
+\tilde{\mathcal R}(\bar{\theta})
\leq \liminf_{n\to\infty}\cosree.
\]
Using the convergence of the values from ~(1),
we conclude that $(\bar{\theta}, \bar{\pi})$ is optimal for the limit problem 
$\cosre$.

\subsection{Proof of Theorem \ref{existence_RipOT}}
We restate here Theorem \ref{existence_RipOT}
\begin{theorem}
\label{thm:353-plain}
Suppose $p\geq q$. Fix $r>0$ and $\varepsilon\ge0$.
For every $\pi\in \M^+(\X\times\Y)$ denote
\[
C(\pi):=\int_{\X\times\Y} yx^{\top}\,d\pi(x,y)\in\R^{q\times p},
\qquad
M(\pi)\ :=\ \frac{r}{\norm{C(\pi)}_{\mathrm F}}\;C(\pi)
\quad(\text{with }M(\pi):=0 \text{ if }C(\pi)=0).
\]
Then, the problem
\[
\mathsf{C}\Ri_r\mathsf{UOT}_{\varepsilon}(\alpha,\beta)
=\inf_{\substack{\pi\in\M^+(\X\times\Y)\\ \norm{M}_{\mathrm F}\le r}}
\Big\{-\int_{\X\times\Y}\langle Mx,y\rangle\,d\pi
\,+\,\D{\varphi_1}{\pi_1}{\alpha}
\,+\,\D{\varphi_2}{\pi_2}{\beta}
\,+\,\varepsilon\,\D{\mathrm{KL}}{\pi}{\alpha\otimes\beta}\Big\}
\]
admits minimizers $(M_\varepsilon^*,\pi_\varepsilon^*)$ with $M_\varepsilon^*=M(\pi_\varepsilon^*)$.
Moreover, if $(M^*,\pi^*)$ minimizes $\mathsf{C}\Ri_r\mathsf{UOT}_{\varepsilon}(\alpha,\beta)$, then
$\pi^*$ minimizes the reduced functional
\[
\mathcal{G}_{\varepsilon}(\pi)
\ :=\ -\,r\,\norm{C(\pi)}_{\mathrm F}
\,+\,\D{\varphi_1}{\pi_1}{\alpha}
\,+\,\D{\varphi_2}{\pi_2}{\beta}
\,+\,\varepsilon\,\D{\mathrm{KL}}{\pi}{\alpha\otimes\beta}
\quad\text{over }\pi\in\M^+(\X\times\Y).
\]
\end{theorem}

\begin{proof}
We prove the Theorem in several steps. The existence part follows from Theorem \ref{app:existence-RUOT}. Let us prove the second part. 
    
\emph{Step 1: Optimal $M$ for fixed $\pi$.}\\
Fix $\pi$ and set $C:=C(\pi)=\int yx^\top\,d\pi$. Using $\langle Mx,y\rangle=\mathrm{tr}(y x^\top M^\top)$,
\[
\int\langle Mx,y\rangle\,d\pi
=\mathrm{tr}\Big(\Big(\int yx^\top\,d\pi\Big)M^\top\Big)
=\langle C,M\rangle_{\mathrm F}.
\]
Hence, for fixed $\pi$, the inner minimization in $M$ is
\[
\inf_{\norm{M}_{\mathrm F}\le r}\ \{ -\,\langle C,M\rangle_{\mathrm F}\}
\ =\ -\,\sup_{\norm{M}_{\mathrm F}\le r}\ \langle C,M\rangle_{\mathrm F}.
\]
By Cauchy–Schwarz in the Frobenius inner-product,
$\sup_{\norm{M}_{\mathrm F}\le r}\langle C,M\rangle_{\mathrm F}=r\,\norm{C}_{\mathrm F}$,
attained at $M=M(\pi)$.
Thus, for every fixed $\pi$,
\[
\inf_{\norm{M}_{\mathrm F}\le r}\Big\{-\int\langle Mx,y\rangle\,d\pi\Big\}
\ =\ -\,r\,\norm{C(\pi)}_{\mathrm F},
\quad\text{with minimizer }M(\pi).
\]

\emph{Step 2: Reduction to a problem on $\pi$.}\\
Plugging the optimal $M(\pi)$ back gives the reduced functional
\[
\mathcal G_\varepsilon(\pi)
:=-\,r\,\norm{C(\pi)}_{\mathrm F}
+\D{\varphi_1}{\pi_1}{\alpha}
+\D{\varphi_2}{\pi_2}{\beta}
+\varepsilon\,\D{\mathrm{KL}}{\pi}{\alpha\otimes\beta}.
\]
Therefore
\[
\mathsf{C}\Ri_r\mathsf{UOT}_{\varepsilon}(\alpha,\beta)\ =\ \inf_{\pi\in\M^+(\X\times\Y)}\ \mathcal G_\varepsilon(\pi),
\]
    
    \emph{Step 3: Conclusion.}\\
    Suppose $(M^*,\pi^*)$ optimal for $\mathsf{C}\Ri_r\mathsf{UOT}_{\varepsilon}(\alpha,\beta)$, we need to prove that $\pi^*$ minimizes $\mathcal{G}_{\varepsilon}$. 
    If we could find $\tilde{\pi}\in \M^+(\X\times \Y)$ s.t. $\mathcal{G}_{\varepsilon}(\tilde{\pi})< \mathcal{G}_{\varepsilon}(\pi^*)$, then  $(M(\tilde\pi),\tilde\pi)$ would give a strictly smaller joint value (by Step 2),
contradicting optimality. Hence $\pi^*$ minimizes the reduced functional.
\end{proof}

\section{Appendix for Section 4}

\subsection{Proof of Theorem \ref{thm_monge_map_RipOT}}
Theorem \ref{thm_monge_map_RipOT} will give us the existence of a Monge map for $\mathsf{C}\Ri_r\mathsf{UOT}$ problems. We state the theorem here for completeness. 
\begin{theorem}\label{thm:monge-CR}
Assume $\X\subset\R^p$, $\Y\subset\R^q$ are compact, $p\ge q$, and either
\begin{enumerate}
\item $\varphi_1,\varphi_2$ are superlinear, strictly convex and the strong compatibility (2) holds; or
\item (balanced case) $\varphi_1=\varphi_2=\iota_{\{1\}}$ and the compatibility (1) holds.
\end{enumerate}
Assume moreover that $\alpha$ is absolutely continuous w.r.t.\ the Lebesgue measure on $\X$.
Then every optimal couple $(M^*,\pi^*)$ for $\mathsf{C}\Ri_r\mathsf{UOT}(\alpha,\beta)$ there exists a map $T_*$ such that
\[
\pi^*=(\mathrm{id},T_*)_\#\pi^*_1 .
\]
If in addition $M^*$ is surjective, there exists a convex Kantorovich potential $f_*\in\Ci(\R^q)$ 
for the linear OT problem on $\R^q$ with cost $c_{\mathrm{ip}}(y',y):=-\langle y',y\rangle$
between $M^*_\#\pi^*_1$ and $\pi^*_2$, differentiable $M^*_\#\pi^*_1$-a.e., such that
\[
T_*\;=\;-\nabla f_*\circ M^*\qquad \pi^*_1\text{-a.e.}
\]
\end{theorem}
In order to prove Theorem 4.2, we need the following results.
\begin{prop}[Fenchel--Kantorovich duality for $-\langle\cdot,\cdot\rangle$]
\label{prop:FK-inner}
Let $\mu,\nu\in\M^+(\R^q)$ be finite measures with compact support (or with finite first moments).
Consider the linear OT problem
\[
\mathrm{OT}^{c_\mathrm{ip}}(\mu,\nu)\ :=\ \inf_{\gamma\in\Pi(\mu,\nu)}
\int_{\R^q\times\R^q} -\langle y',y\rangle\,\mathrm{d}\gamma(y',y).
\]
Then the Kantorovich dual is
\[
\mathrm{OT}^{c_\mathrm{ip}}(\mu,\nu)\;=\;
\sup_{f\in \Gamma(\R^q)}\ \Big\{
-\!\int_{\R^q} f(y')\,\mathrm{d}\mu(y')\;
-\!\int_{\R^q} f^*(y)\,\mathrm{d}\nu(y)\Big\},
\]
where $\Gamma(\R^q)$ denotes proper l.s.c.\ convex functions and 
$f^*$ is the convex conjugate of $f$.
Moreover, dual optimizers exist and there is no duality gap.
\end{prop}

\begin{proof}
By Fenchel--Young, inequality for every $f$ and every $(y',y)$,
\(
f(y')+f^*(y)\ \ge\ \langle y',y\rangle
\iff
-\langle y',y\rangle\ \le\ -f(y')-f^*(y).
\)
Integrating now against any $\gamma\in\Pi(\mu,\nu)$ we get
\[
\int -\langle y',y\rangle\,\mathrm{d}\gamma
\ \le\
-\int f\,\mathrm{d}\mu - \int f^*\,\mathrm{d}\nu.
\]
Taking the infimum in $\gamma$ and the supremum in $f$ yields weak duality. 
Under the stated compactness assumption, the standard Kantorovich duality theorem applies to the l.s.c.\ cost $-\langle \cdot,\cdot\rangle$. This follows from Theorem Fenchel–Moreau on $\M(\R^q)\times\Ci(\R^q)$.
\end{proof}

\begin{coroll}[Optimality/KKT conditions]
\label{cor:KKT-inner}
Let $\gamma^*\in\Pi(\mu,\nu)$ and $f_*\in\Gamma(\R^q)$ be primal/dual optimizers for Proposition~\ref{prop:FK-inner}. Then:
\begin{enumerate}
\item \emph{Support condition}
\[
\mathrm{spt}\,\gamma^*\ \subset\ \{(y',y)\in\R^q\times\R^q:\ y\in -\partial f_*(y')\}.
\]
Equivalently, $y'\in \partial f_*^*(-y)$ on $\mathrm{spt}\,\gamma^*$.
\item \emph{Measurable selection:} There exists a measurable map $T^\sim:\R^q\to\R^q$ with 
$\gamma^*=(\mathrm{id},T^\sim)_\#\mu$ and $T^\sim(y')\in -\partial f_*(y')$ $\mu$-a.e.
\item \emph{Gradient form (a.e.\ differentiability):} if $\mu$ is absolute continuous w.r.t Lebesgue measure, then $f_*$ is differentiable $\mu$-a.e.\ and 
\[
\gamma^*=(\mathrm{id},-\nabla f_*)_\#\mu,\qquad T^\sim(y')=-\nabla f_*(y')\quad \mu\text{-a.e.}
\]
\end{enumerate}
\end{coroll}

\begin{proof}
Optimality forces equality in Fenchel--Young  $\gamma^*$-a.e., i.e., 
$f_*(y')+f_*^*(y)=\langle y',y\rangle$, which is equivalent to $y\in -\partial f_*(y')$.
This gives (1). Disintegrating $\gamma^*$ w.r.t.\ $\mu$ and choosing a measurable selector from the monotone set $-\partial f_*$ yields (2). If $\mu$ absolute continues with respect to Lebesgue measure, Alexandrov/Rademacher imply $f_*$ is a.e.\ differentiable and the subgradient is single-valued a.e., giving (3).
\end{proof}

\paragraph{Proof of Theorem \ref{thm_monge_map_RipOT}}
We prove the theorem in several steps.
\textit{Step 1: Reduce the coupling to $\R^q\times\R^q$.}
By the existence theorem for $\mathsf{C}\Ri_r\mathsf{UOT}$ (Theorem~\ref{app:existence-RUOT}),
there exists an optimal pair $(M^*,\pi^*)$.
We set the measures on $\R^q$
\[
\mu\ \triangleq\ M^*_\#\pi^*_1,
\qquad
\nu\ \triangleq\ \pi^*_2.
\]
Consider the pushforward plan on $\R^q\times\R^q$ defined by
\[
\gamma^*\ \triangleq\ (M^*,\mathrm{id})_\#\pi^* .
\]
Then $\gamma^*\in\Pi(\mu,\nu)$ and, by change of variables,
\begin{equation}\label{eq:star}
\int_{\X\times\Y}-\langle M^*x,y\rangle\,d\pi^*(x,y)
\;=\;\int_{\R^q\times\R^q} -\langle y',y\rangle\,d\gamma^*(y',y).
\end{equation}
Optimality of $(M^*,\pi^*)$ implies that, for fixed $M^*$,
$\pi^*$ minimizes the $\mathsf{UOT}$ problems with cost $c_{M^*}(x,y)=-\langle M^*x,y\rangle$.
Hence, by \eqref{eq:star}, $\gamma^*$ is optimal for the \emph{linear} OT problem on $\R^q$ between $\mu$ and $\nu$
with cost $c_{\mathrm{ip}}(y',y)=-\langle y',y\rangle$.

\smallskip
\textit{Step 2: Duality on $\R^q$ and graph structure.}
With $\mu=M^*_\#\pi^*_1$ and $\nu=\pi^*_2$ from \ref{cor:KKT-inner}, the pushed-forward optimal plan 
$\gamma^*=(M^*,\mathrm{id})_\#\pi^*$ solves $\mathrm{OT}_{\mathrm{ip}}(\mu,\nu)$.
By Corollary~\ref{cor:KKT-inner}, $\gamma^*$ is a graph $(\mathrm{id},T^\sim)_\#\mu$ with 
$T^\sim\in -\partial f_*$. Lifting back to $\X$ via $y'=M^*x$ gives 
$\pi^*=(\mathrm{id},T_*)_\#\pi_1^*$ with $T_*(x)=T^\sim(M^*x)$, and if $\mu$ absolute continuous with respect to Lebesgue measure then 
$T_*(x)=-\nabla f_*(M^*x)$ $\pi_1^*$-a.e.

The Kantorovich dual for $c_{\mathrm{ip}}(y',y)=-\langle y',y\rangle$ is 
\[
\sup_{f\in\Ci(\R^q)}\; \Big\{-\int f(y')\,d\mu(y') - \int f^*(y)\,d\nu(y)\Big\}.
\]
Let $f_*$ be an optimal potential.
By Fenchel optimality, $\gamma^*$ is concentrated on the set
\[
\mathcal G_*  \triangleq \ \{(y',y)\in\R^q\times\R^q \;:\; y\in -\partial f_*(y')\}\,,
\]
i.e.\ $y\in -\partial f_*(y')$ $\mu$-a.e.\ (equivalently, $y' \in \partial f_*^*(-y)$).
In particular, there exists $T^\sim:\R^q\to\R^q$ with
\[
\gamma^* = (\mathrm{id},T^\sim)_\#\mu
\qquad\text{and}\qquad
T^\sim(y')\in -\partial f_*(y')\quad \mu\text{-a.e.}
\]
If $f_*$ is differentiable $\mu$-a.e.\ (this will be the case when $\mu$ is a.c.\ on $\R^q$),
then $T^\sim(y')=-\nabla f_*(y')$ $\mu$-a.e.

\smallskip
\textit{Step 3: Lift the graph back to $\X$.}
Define $T_*:\X\to\Y$ by
\[
T_*(x)\ :=\ T^\sim\big(M^*x\big).
\]
Then
\[
(M^*,\mathrm{id})_\#\big((\mathrm{id},T_*)_\#\pi^*_1\big)
= (\mathrm{id},T^\sim)_\#\big(M^*_\#\pi^*_1\big)
= (\mathrm{id},T^\sim)_\#\mu
= \gamma^*.
\]
But $(M^*,\mathrm{id})_\#\pi^*=\gamma^*$.
Since disintegration of measures with respect to the map $x\mapsto M^*x$ is unique up to $\pi^*_1$-null sets,
and $y$ under an optimal plan on $\R^q\times\R^q$ depends only on $y'=M^*x$, it follows that
$\pi^*=(\mathrm{id},T_*)_\#\pi^*_1$.

\smallskip
\textit{Step 4: Surjective case and differentiability.}
If $M^*$ is surjective and $\alpha$ is absolute continuous w.r.t the Lebesgue measure, then $\pi^*_1\ll\alpha$ in both the balanced case
($\pi^*_1=\alpha$) and in the unbalanced case (first-order optimality gives $\pi^*_1$ absolute continuous with respect to $\alpha$ with continuous density).
Hence $\mu=M^*_\#\pi^*_1$ is absolutely continuous w.r.t.\ Lebesgue measure on $\R^q$.
By Alexandrov theorem, the optimal potential $f_*$ is differentiable $\mu$-a.e.,
and the optimal $\gamma^*$ is induced by the map $y'\mapsto -\nabla f_*(y')$.
Therefore, $T^\sim=-\nabla f_*$ $\mu$-a.e., and the representation from Step~3 yields
\[
T_*(x)=T^\sim(M^*x) = -\,\nabla f_*\big(M^*x\big)\qquad \pi^*_1\text{-a.e.,}
\]
as claimed.
\subsection{Proof of Theorem \ref{thm_entropic_map_conv}}
Clarified the existence of a Monge map for the problem \(\mathsf{C\Ri_{\mathit{r}}UOT}(\alpha,\beta)\), we turn to the task of its approximation using entropic regularization to leverage the computational advantages.

In the following, for every \(\varepsilon>0\), we will denote \((\pi^{\varepsilon}, M^{\varepsilon})\) an optimal couple for \(\mathsf{\Ri_{\mathit{r}}UOT}_{\bphi,\varepsilon}(\alpha,\beta)\) s.t. \(M^{\varepsilon} = M(\pi^{\varepsilon})\) and we name \(\alpha^{\varepsilon} := \pi^{\varepsilon}_1\) and \(\beta^{\varepsilon} := \pi^{\varepsilon}_2\).
Observe that \(\alpha^{\varepsilon}\) has support in \(\X\), indeed \(\alpha\) has support in \(\X\) and \(\pi^{\varepsilon}\ll \alpha\otimes \beta\) imples \(\alpha^{\varepsilon}\ll \alpha\).
Moreover, it will be useful to note that, since \(\Vert M x\Vert \leq r\max_{x\in \X}\Vert x\Vert\) for every \(M\in \F_r\) and \(x\in \X\), the measure \(M^{\varepsilon}_{\#}\alpha^{\varepsilon}\) has support contained in the compact ball \(B_r := \{y\in \R^q \, \| \, \Vert y\Vert \leq r\max_{x\in \X}\Vert x\Vert\}\) for every \(\varepsilon>0\).

We fix two sequences \((\varepsilon_n)_{n\in\N},(\varepsilon_j')_{j\in\N}\subset (0,+\infty)\) s.t. \(\varepsilon_n,\varepsilon_j'\to 0\). 

\begin{defn}
    For every \(j,n\in \N\) we define the \emph{entropic map} \(T_{n,j} : \R^p \to \R^q\) as follows
    \[T_{n,j}(x) = \frac{\int_{\Y} y \exp\left[ \frac{1}{\varepsilon_n} ( g_{j,n}(y) + \< M^{\varepsilon_j'}x,y\>) \right] \d \beta^{\varepsilon_j'}(y)}{\int_{\Y} \exp\left[ \frac{1}{\varepsilon_n} ( g_{j,n}(y) + \< M^{\varepsilon_j'}x,y\>) \right] \d \beta^{\varepsilon_j'}(y)},\]
    where \((f_{j,n},g_{j,n})\in \Ci(B_r)\times \Ci(\Y)\) are optimal for \(\mathsf{D}_{\varepsilon_n}^{c_{\mathrm{ip}}}(M^{\varepsilon_j'}_{\#}\alpha^{\varepsilon_j'},\beta^{\varepsilon_j'})\), where

    \begin{equation*}
\begin{split}
\mathsf{D}_{\varepsilon_n}^{c_{\mathrm{ip}}}(M^{\varepsilon_j'}_{\#}\alpha^{\varepsilon_j'},\beta^{\varepsilon_j'}) &= \sup_{f,g \in \mathcal{C}(B_r) \times \mathcal{C}(\mathcal{Y})} \;  
   \int_{B_r} f \, d M^{\varepsilon_j'}_{\#}\alpha^{\varepsilon_j'} 
 + \int_{\mathcal{Y}} g \, d\beta^{\varepsilon_{j}'} \\
& - \varepsilon_n \int_{B_r \times \mathcal{Y}} 
   \left[ 
     \exp\left(\frac{f \oplus g - c_{\mathrm{ip}}}{\varepsilon_n}\right) - 1
   \right] 
   d\big(M^{\varepsilon_{j}'}_{\#}\alpha^{\varepsilon_{j}'} \otimes \beta^{\varepsilon_{j}'}\big).
\end{split}
\end{equation*}
\end{defn}

Note that, in our setting, the hypothesis of Theorem \ref{thm:conv-entr-minimizers} are satisfied, hence we can find a subsequence \((\varepsilon_{j_h}')_{h\in \N}\), independend of \(n\), s.t. \(\pi^{\varepsilon_{j_h}'}\wto \pi^*\) and \(M^{\varepsilon_{j_h}'} \to M^*\) with \((\pi^*,M^*)\) optimal for \(\mathsf{\Ri_{\mathit{r}}UOT}_{\bphi}(\alpha,\beta)\). 
In particular, denoting for every \(A\in \R^{q\times p}\) the cost \(c_A(x,y) = -\< Ax,y\>\), we have \(c_{M^{\varepsilon_{j_h}'}}\to c_{M^*}\) uniformly. We name \(\alpha^*\) and \(\beta^*\) the marginals of \(\pi^*\).
We have \(M^{\varepsilon_{j_h}'}_{\#}\alpha^{\varepsilon_{j_h}'} \wto M^*_{\#}\alpha^*\), indeed by weak convergence the family of measures \((\alpha^{\varepsilon_{j_h}'})_{h\in\N}\), and consequently also \((M^{\varepsilon_{j_h}'}_{\#}\alpha^{\varepsilon_{j_h}'})_{h\in\N}\), is bounded, therefore 
it suffices to prove
\[\int_{\R^q} \phi \d M^{\varepsilon_{j_h}'}_{\#}\alpha^{\varepsilon_{j_h}'} \to \int_{\R^q} \phi \d M^*_{\#} \alpha^*\]
for every \(\phi\in \Ci_b(\R^q)\) Lipschitz continuous (\citep[Theorem 13.16]{klenke2008probability}). Fix \(\phi \in \Ci_b(\R^q)\) Lipschitz continuous and note that actually 
\[\begin{split}
        \left\| \int_{\R^q} \phi \d M^{\varepsilon_{j_h}'}_{\#}\alpha^{\varepsilon_{j_h}'} - \int_{\R^q} \phi \d M^*_{\#} \alpha^*\right\| &\leq \int_{\X} \| \phi(M^{\varepsilon_{j_h}'} x) - \phi(M^* x)\| \d \alpha^{\varepsilon_{j_h}'}(x)\\
        &\qquad \qquad \qquad \qquad + \left\vert \int_{X} \phi(M^* x) \d \left( \alpha^{\varepsilon_{j_h}'} - \alpha^* \right)(x) \right\vert\\
        &\leq L_{\phi} m(\alpha^{\varepsilon_{j_h}'}) \Vert M^{\varepsilon_{j_h}'} - M^*\Vert_F \max_{x\in \X} \Vert x\Vert\\
        &\qquad \qquad \qquad \qquad + \left\vert \int_{\X} \phi(M^* x) \d \left( \alpha^{\varepsilon_{j_h}'} - \alpha^* \right)(x) \right\vert\\
        &\to 0
\end{split}\]
where \(L_{\phi}\) is the Lipschitz constant of \(\phi\).

\begin{prop}\label{prop_entr_map_reg_ip}
        For every \(n\in\N\) define \(T_n : \X\to \Y\) as 
        \[T_n(x) = \frac{\int_{\Y} y \exp\left[ \frac{1}{\varepsilon_n}(g_n(y) + \< M^*x,y\> ) \right] \d \beta^*(y)}{\int_{\Y} \exp\left[ \frac{1}{\varepsilon_n}(g_n(y) + \< M^*x,y\> ) \right] \d \beta^*(y)},\]
        for some \((f_n,g_n)\in \Ci(B_r)\times \Ci(\Y)\) optimal for \(\mathsf{D}_{\varepsilon_n}^{c_{\mathrm{ip}}}(M^*_{\#}\alpha^*,\beta^*)\). Then \(T_{j_h,n}\to T_n\) in \(L^2(\alpha^*)\) for every \(n\in\N\).
\end{prop}

\begin{proof}
    
From the previous discussion we know that 
    \(M^{\varepsilon_{j_h}'}_{\#}\alpha^{\varepsilon_{j_h}'} \rightharpoonup M^*_{\#}\alpha^*\) 
    and \(\beta^{\varepsilon_{j_h}'} \rightharpoonup \beta^*\) as \(h \to \infty\).
    Fix \(x_0 \in \X\) and note that, up to replacing \((f_{j,n},g_{j,n})\) by 
    \((f_{j,n} - f_{j,n}(x_0), \, g_{j,n} + f_{j,n}(x_0))\), we may assume 
    \(f_{j,n}(x_0) = 0\) for every \(j,n \in \N\).

    In particular, by the compactness argument in \citep{genevay2018learning}, 
    we can find \((f_n,g_n) \in \Ci(B_r) \times \Ci(\Y)\) optimal for 
    \(\mathsf{D}_{\varepsilon_n}^{c_{\mathrm{ip}}}(M^*_{\#}\alpha^*,\beta^*)\) and a subsequence
    \((\tilde{j}_h)_{h \in \N}\) of \((j_h)_{h \in \N}\) such that
    \(f_{\tilde{j}_h,n} \to f_n\) and \(g_{\tilde{j}_h,n} \to g_n\) uniformly on their domains
    as \(h \to \infty\), for every fixed \(n \in \N\).

    To ease notation, for every \(x \in \X\), \(n \in \N\), \(g \in \Ci(\Y)\) 
    and \(A \in \R^{q \times p}\) define
    \[
        F_n^x(g,A)(y) 
        = \exp\left( \frac{1}{\varepsilon_n} \bigl(g(y) - c_A(x,y)\bigr) \right),
        \qquad y \in \Y,
    \]
    where \(c_A(x,y) = -\langle Ax, y \rangle\) is the inner-product cost associated to \(A\).
    Observe that
    \[
    \begin{split}
        \norm{ F_n^x(g_{\tilde{j}_h,n}, M^{\varepsilon_{\tilde{j}_h}'}) 
            - F_n^x(g_n, M^*)}_{\infty}
        &\leq 
        \omega_n \left( \norm{g_{\tilde{j}_h,n} - g_n}_{\infty} 
            + \norm{ c_{M^{\varepsilon_{\tilde{j}_h}'}} - c_{M^*}}_{\infty} \right)
        \\
        &\longrightarrow 0 \quad \text{as } h \to \infty,
    \end{split}
    \]
    where \(\omega_n\) is the modulus of continuity of 
    \(t \mapsto \exp\bigl( t / \varepsilon_n \bigr)\) on the compact interval
    where the uniformly bounded functions 
    \(g_{\tilde{j}_h,n} - c_{M^{\varepsilon_{\tilde{j}_h}'}}\) and 
    \(g_n - c_{M^*}\) take their values.

    Hence, for every \(\phi \in \Ci(\Y)\) and \(n \in \N\),
    \[
    \begin{split}
        &\left\|
            \int_{\Y} \phi \, F_n^x(g_{\tilde{j}_h,n}, M^{\varepsilon_{\tilde{j}_h}'})
                \,\mathrm{d}\beta^{\varepsilon_{\tilde{j}_h}'} 
            - \int_{\Y} \phi \, F_n^x(g_n, M^*) \,\mathrm{d}\beta^*
        \right\|
        \\
        &\quad \leq 
        \norm{\phi}_{\infty} 
        \norm{F_n^x(g_{\tilde{j}_h,n}, M^{\varepsilon_{\tilde{j}_h}'}) 
            - F_n^x(g_n, M^*)}_{\infty} 
        \, \beta^{\varepsilon_{\tilde{j}_h}'}(\Y)
        \\
        &\qquad 
        + \left\| 
            \int_{\Y} \phi \, F_n^x(g_n, M^*) 
            \,\mathrm{d}(\beta^{\varepsilon_{\tilde{j}_h}'} - \beta^*)
        \right\|
        \longrightarrow 0
    \end{split}
    \]
    as \(h \to \infty\), since \(\phi F_n^x(g_n,M^*) \in \Ci(\Y)\) 
    and \(\beta^{\varepsilon_{\tilde{j}_h}'} \rightharpoonup \beta^*\).
    
    Consequently, we deduce that \(T_{\tilde{j}_h,n}(x) \to T_n(x)\) pointwise for all \(x \in \X\).
    Moreover, by Jensen's inequality,
    \[
        \norm{T_{\tilde{j}_h,n}(x)}^2, \ \norm{T_n(x)}^2 
        \leq \max_{y \in \Y} \norm{y}^2 < +\infty
        \quad \text{for all } x \in \X,
    \]
    so by the dominated convergence theorem we obtain
    \(T_{\tilde{j}_h,n} \to T_n\) in \(L^2(\alpha^*)\).

    Finally, note that the above argument can be applied to every subsequence
    of \((T_{j_h,n})_{h \in \N}\), and that \(T_n\) is independent of the particular choice
    of optimal potential \(g_n\) in its expression (all such \(g_n\) differ only by an 
    additive constant, see e.g. \cite{genevay2018learning}). 
    Therefore the whole sequence satisfies
    \[
        T_{j_h,n} \to T_n \quad \text{in } L^2(\alpha^*)
    \]
    for every fixed \(n \in \N\).
    \end{proof}
We restate here Theorem \ref{thm_entropic_map_conv}.
\begin{theorem}
        Assume $\X\subset \R^p$, $\Y\subset \R^q$ compact domains, \(p\geq q\), \(\varphi_1,\varphi_2\) superlinear strictly convex satisfying \eqref{comp_cond_strong} or $\varphi_1 = \varphi_2 = \iota_{\{1\}}$ satisfying \eqref{comp_cond} and \(\alpha\) absolutely continuous w.r.t. the Lebesgue measure on \(\X\). Assume also $B_r\subset \Y$. Then there exists a subsequence $(\varepsilon_{j_h}')_{h\in\N}$ s.t. $M^{\varepsilon_{j_h}'}\to M^*$ optimal for \(\mathsf{C\Ri_{\mathit{r}}UOT}_{\bphi}(\alpha,\beta)\). 
        Moreover, suppose \(M^*\) surjective and that \( M^*_{\#}\alpha^*\) and \(\beta^*\) satisfy the Assumptions A1-A3 in \citep{pooladian2021entropic} for $\alpha\geq 2$. Then
        \[\lim_{n\to +\infty} \lim_{h\to +\infty} T_{j_h,n} = T_* \quad \text{in \(L^2(\alpha^*)\),}\]
        where \(T_*\) is a Monge map for \(\mathsf{C\Ri_{\mathit{r}}UOT}(\alpha,\beta)\) which pushes \(\alpha^*\) to \(\beta^*\).
\end{theorem}
\begin{proof}
        From Proposition \ref{prop_entr_map_reg_ip} we have \(T_{j_h,n}\to T_n\) in \(L^2(\alpha^*)\) for every \(n\in\N\). From Theorem \ref{thm:monge-CR} we know that the Monge map for 
        \(\mathsf{\Ri_{\mathit{r}}UOT}_{\bphi}(\alpha,\beta)\) associated to the minimisers \((M^*,\pi^*)\) is \(T_* = -\nabla f_* \circ M^*\) where \(f_*\in \Ci(B_r)\) is an optimal Kantorovich potential for \(\mathsf{OT}^{c_{\mathrm{ip}}}(M^*_{\#}\alpha^*,\beta^*)\). 
        In particular, \(-\nabla f_*\) is the Monge map for the problem \(\mathsf{OT}^{c_{\mathrm{ip}}}(M^*_{\#}\alpha^*,\beta^*)\), therefore \(-\nabla f_* = \nabla \phi\) \(M^*_{\#}\alpha^*\)-a.e. by the hypothesis.
        Moreover, it is easy to see that \(T_n = - \nabla f_n \circ M^*\) with \((f_n,g_n)\) optimal for \(\mathsf{D}_{\varepsilon_n}^{c_{\mathrm{ip}}}(M^*_{\#}\alpha^*,\beta^*)\) (see Theorem 2.7 and 3.16 in \cite{carlier2017convergence} for the convergence of
entropic OT potentials to Kantorovich potentials in the balanced case with cost \(c_{\mathrm{ip}}\)). 
        The claim follows by applying Corollary 1 in \citep{pooladian2021entropic}, indeed 
        \[\begin{split}
                \int_{\X} \Vert T_n - T_*\Vert^2 \d \alpha^* &= \int_{\X} \Vert - \nabla f_n \circ M^* + \nabla f_* \circ M^*\Vert^2 \d \alpha^*\\
                &= \int_{\X} \Vert - \nabla f_n - \nabla \phi  \Vert^2 \d M^*_{\#}\alpha^* \to 0 \quad \text{as \(n\to +\infty\).}
        \end{split}\]
\end{proof}

\section{Appendix for Section 5}

To approximate a solution for Problem $\mathsf{C\Ri_r UOT}_{\varepsilon}$ we use the following block coordinate descent algorithm \citep{bertsekas1997nonlinear}:
\begin{equation*}
        \begin{split}
                \pi^{k+1} &= \argmin_{\pi\in \M^+(\X\times \Y)} -\int_{\X\times \Y} \< M_{k} x, y\> \d\pi(x,y) + \D{\varphi_1}{\pi_1}{\alpha} + \D{\varphi_2}{\pi_2}{\beta} + \varepsilon \D{\mathrm{KL}}{\pi}{\alpha\otimes \beta}\\
                M_{k+1} &= \argmin_{M\in \F_r} -\int_{\X\times \Y} \< M x, y\> \d\pi^{k+1}(x,y),
        \end{split}
\end{equation*}
that adapting Lemma 4.2.2 in \citep{vayer2020contribution} the algorithm becomes:
\begin{equation}\label{eq_BCD}
        \begin{split}
                \pi^{k+1} &= \argmin_{\pi\in \M^+(\X\times \Y)} -\int_{\X\times \Y} \< M_{k} x, y\> \d\pi(x,y) + \D{\varphi_1}{\pi_1}{\alpha} + \D{\varphi_2}{\pi_2}{\beta} + \varepsilon \D{\mathrm{KL}}{\pi}{\alpha\otimes \beta}\\
                M_{k+1} &=  \frac{r}{\left\Vert \int_{\X\times \Y} y x^T \d \pi^{k+1}(x,y)\right\Vert_{F}} \int_{\X\times \Y} y x^T \d \pi^{k+1}(x,y).
        \end{split}
\end{equation}

\subsection{Proof of Theorem \ref{thm:conv-algo}}
\begin{theorem}[Theorem \ref{thm:conv-algo}]
        Let \(\X = \{x_i\}_{i=1}^n\subset \R^p\), \(\Y = \{y_j\}_{j=1}^m\subset \R^q\), \(\alpha=\sum_{i=1}^n a_i \delta_{x_i}\), \(\beta=\sum_{j=1}^m b_j \delta_{y_j}\) and \(\varepsilon,r>0\). Suppose \(\{a_i\}_{i=1}^n, \{b_j\}_{j=1}^m\subset (0,+\infty)\) and the entropy functions \(\varphi_1\) and \(\varphi_2\) to be superlinear.
        Then, any limit point of the sequence \(((M_k,\pi^k))_{k\in\N}\) defined by the block coordinate descent scheme \eqref{eq_BCD} is a stationary point of the objective function of \(\mathsf{C\mathcal{R}_r UOT}_{\varepsilon}(\alpha,\beta)\).
\end{theorem}
\begin{proof}
    
The objective function of \(\mathsf{C\mathcal{R}_r UOT}_{\varepsilon}(\alpha,\beta)\) can seen as the function \(J : \R^{n\times m} \times \R^{q\times p} \to \R\) s.t.
        \[\begin{split}
                J (P,M) = - \sum_{i=1}^n\sum_{j=1}^m \< M x_i, y_j\> P_{i,j} + &\sum_{i=1}^n \varphi_1 \left( \frac{P_i}{a_i} \right) a_i + \sum_{j=1}^m \varphi_2\left( \frac{P^j}{b_j} \right) b_j\\
                & + \varepsilon \sum_{i=1}^n\sum_{j=1}^m \left[ \frac{P_{i,j}}{a_i b_j} \log \frac{P_{i,j}}{a_i b_j} - \frac{P_{i,j}}{a_i b_j} + 1 \right] a_i b_j\\
                & + \delta_{[0,+\infty)^{n\times m}}(P) + \delta_{\F_r}(M)
        \end{split}\]
        where we identify any plan \(\pi\in \M^+(\X\times \Y)\) with the matrix \(P = (\pi(\{(x_i,y_j)\}))_{i,j}\in [0,+\infty)^{n\times m}\) and we denote \(P_i = \sum_{j=1}^m P_{i,j}\), \(P^j = \sum_{i=1}^n P_{i,j}\) for every \(i\) and \(j\). 
       
Let us now rewrite the the function $J$ as
\[
J(P,M)= \langle C(M), P\rangle
+ \sum_i a_i \phi_1(P_i/a_i)
+ \sum_j b_j \phi_2(P^j/b_j)
+ \varepsilon\,\mathrm{KL}(P,ab^\top)
+ \delta_{[0,+\infty)^{n\times m}}(P)
+ \delta_{F_r}(M),
\]
where $C(M)=\bigl(y_j x_i^\top M\bigr)_{i,j}$. So, we are able to decompose $J$ as $
J(P,M)= J_0(P,M) + g(P) + h(M)$,
where
\[
J_0(P,M)= \langle C(M),P\rangle
+ \sum_i a_i\phi_1(P_i/a_i)
+ \sum_j b_j\phi_2(P^j/b_j)
+ \varepsilon\,\mathrm{KL}(P,ab^\top),
\]
\[
g(P)=\delta_{[0,+\infty)^{n\times m}}(P),
\qquad
h(M)=\delta_{F_r}(M).
\]
The smooth part $J_0$ is continuously differentiable, while $g$ and $h$ are proper, convex and lower semicontinuous. 

Because $\phi_1,\phi_2$ are superlinear and the KL term controls the total mass of $P$, the quantity $\sum_{i,j} P_{ij}$ is uniformly bounded on every sublevel set of $J$. Since $F_r$ is compact and $[0,+\infty)^{n\times m}$ is closed, all sublevel sets of $J$ are compact. Now for fixed $M$, the subproblem
\[
\min_{P\ge 0}\ J_0(P,M)
\]
is strictly convex because of the KL regularization and the superlinear functions
$\phi_1,\phi_2$. As a result, it has a unique minimizer $P^\star(M)$.
Moreover, for $\varepsilon>0$ the minimizer satisfies $P_{ij}^\star(M)>0$ for all $i,j$, so $g(P)$ does not play a role in the opimization and $J_0$ is differentiable at $P^\star(M)$.

Now we need to consider the exact minimization in the $M$--block. For fixed $P$, the subproblem
\[
\min_{M\in F_r} \ \langle C(M),P\rangle
\]
is linear over the Frobenius ball $F_r=\{M:\norm{M}_F\le r\}$. Let
$C(P)=\sum_{i,j} P_{ij}\, x_i y_j^\top$. If $C(P)\neq 0$, the unique minimizer is 
$M(P)= r \,\frac{C(P)}{\norm{C(P)}_F}.$
If $C(P)=0$, every element of $F_r$ is optimal, in order to keep the block-coordinate map single-valued we set $M(P)=0$ in this case. Thus the $M$--update is always uniquely defined. Now, we get the convergence by applying Lemma 3.1 and Theorem 4.1 in \cite{tseng2001convergence}.
The function $J$ has the form $J=f+g+h$ with
\[
f=J_0 \ \text{(smooth)},\qquad
g(P)=\delta_{[0,+\infty)^{n\times m}}(P),\qquad
h(M)=\delta_{F_r}(M).
\]
Each block subproblem is solved exactly, and all sublevel sets of $J$ are compact. Therefore, every limit point of the alternating minimization sequence $(P_k,M_k)$ is a stationary point of $J$. This concludes the proof.

\end{proof}

The practical pseudocode implementation of the alternate minimization scheme \eqref{eq_BCD} in the discrete case is the following:

\   

\begin{algorithm}[H]\label{bcd_alg}
\caption{BCD for \(\mathsf{\mathcal{R}_r UOT}_{\bphi,\varepsilon}(\alpha,\beta)\)}
\KwIn{Entropy functions \(\varphi_1, \varphi_2\), numbers \(\varepsilon,r>0\), source \(\alpha=\sum_{i=1}^n a_i \delta_{x_i}\) and target \(\beta=\sum_{j=1}^m b_j \delta_{y_j}\) with \(\X = \{x_i\}_{i=1}^n\subset \R^p\), \(\Y = \{y_j\}_{j=1}^m\subset \R^q\), \(a = (a_i)_{i=1}^n\in (0,+\infty)^n\), \(b=(b_j)_{j=1}^m \in (0,+\infty)^m\)}
\KwOut{\(M^{\varepsilon}\), \(P^{\varepsilon}\) optimal for \(\mathsf{\mathcal{R}_r UOT}_{\bphi,\varepsilon}(\alpha,\beta)\)}

$M \gets M_0$\;
\(k\gets 1\)\;
\While{\(k\leq N\) \textrm{and} \(\mathrm{err} \leq \mathrm{tol}\)}{
  \(C \gets (-\< M x_i, y_j\>)_{i,j}\)\;
  \(P \gets \mathrm{Sinkhorn}(a, b, C, \varphi_1, \varphi_2, \varepsilon)\)\;
  \(M \gets \frac{k}{\Vert \sum_{i,j} y_jx_i^T P_{i,j}\Vert_F}\sum_{i,j} y_jx_i^T P_{i,j}\)\;
}
\Return{\(M, P\)}

\end{algorithm}

\section{ADDITIONAL EXPERIMENTS}
In this section, we provide further experiments to evaluate the effectiveness of our Algorithm \ref{bcd_alg}.
First, Figure \ref{fig:toy-ex} provides further insights to better grasp the effect of unbalancedness on the entropic map. We gradually increase the admitted unbalancedness by decreasing the parameter $\lambda$. 

\begin{figure}[h]
\centering

\begin{subfigure}{0.4\textwidth}
    \includegraphics[width=\linewidth, trim=0 100 0 150, clip]{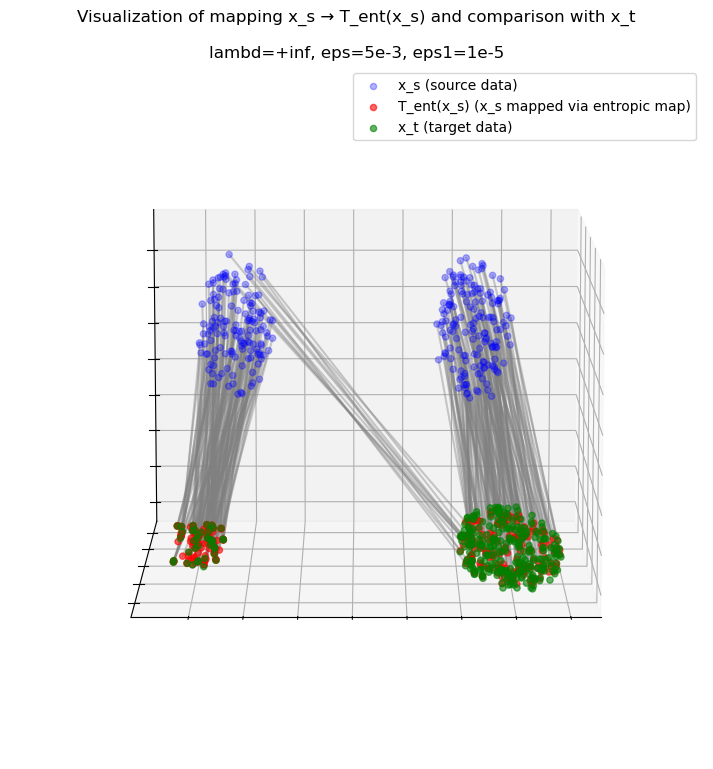}
    \caption{Balanced, $\lambda = +\infty$}
\end{subfigure}
\begin{subfigure}{0.4\textwidth}
    \includegraphics[width=\linewidth, trim=0 100 0 150, clip]{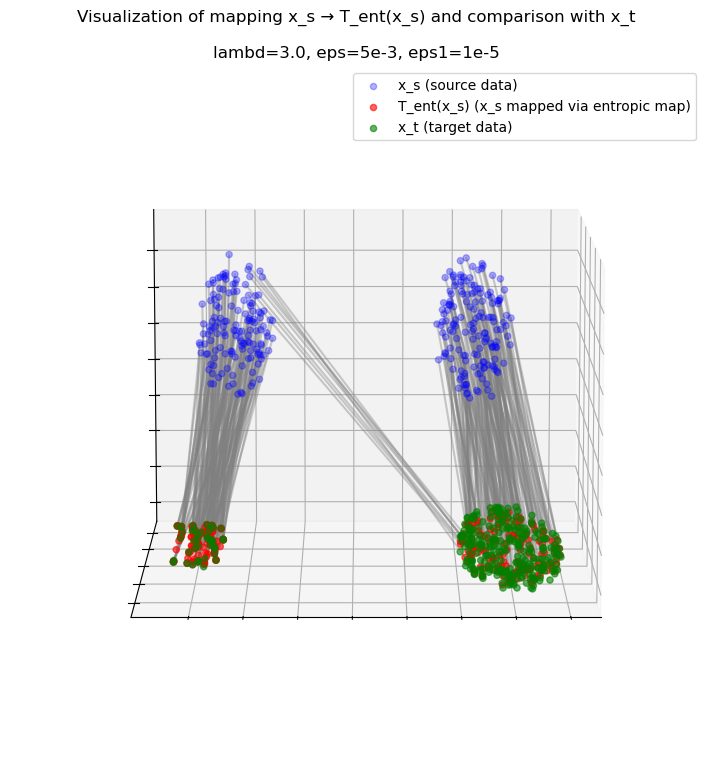}
    \caption{$\lambda = 3.0$}
\end{subfigure}

\begin{subfigure}{0.4\textwidth}
    \includegraphics[width=\linewidth, trim=0 100 0 150, clip]{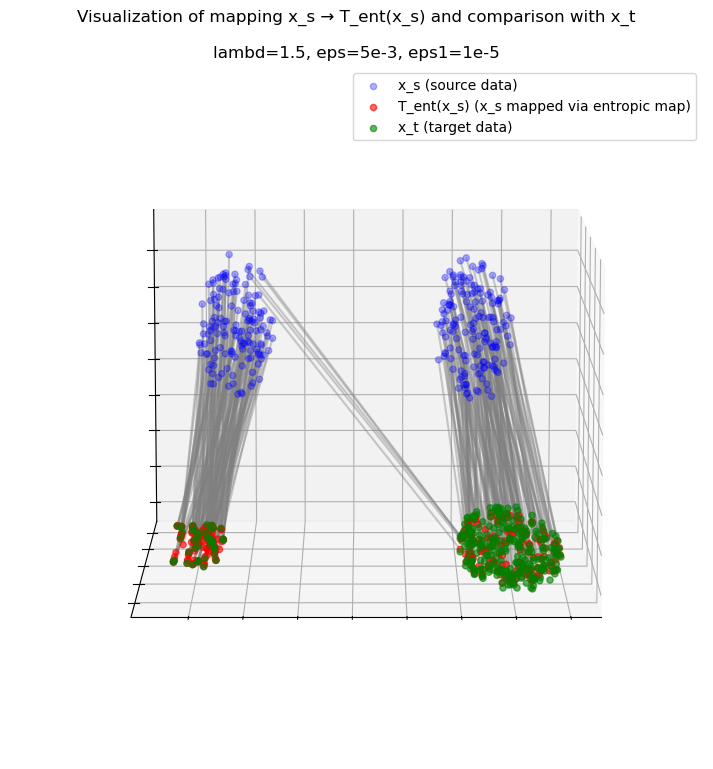}
    \caption{$\lambda = 1.5$}
\end{subfigure}
\begin{subfigure}{0.4\textwidth}
    \includegraphics[width=\linewidth, trim=0 100 0 150, clip]{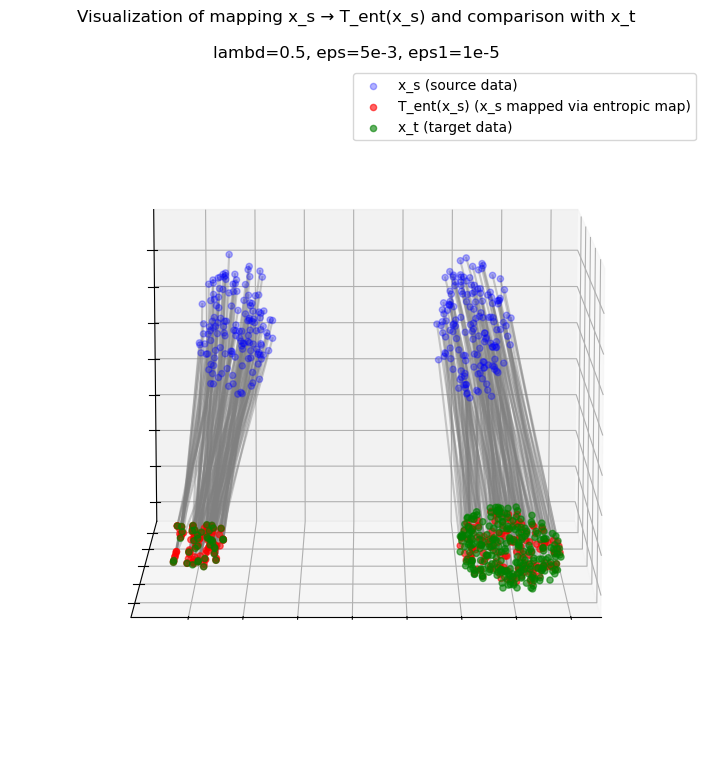}
    \caption{$\lambda = 0.5$}
\end{subfigure}

\caption{The source data (blue) is generated sampling from a balanced mixture of uniform distributions on two ellipsoids in 3D, while the target data (green) is obtained by sampling from an unbalanced mixture of the uniform distribution on a square $\mathcal{S}$ and the uniform distribution on an ellipse $\mathcal{E}$ in 2D, precisely the latter mixture is $\beta = 0.85 \mathcal{E} + 0.15\mathcal{S}$. For visualization purposes we lift $\R^2$ into $\R^3$ by padding the third coordinate to zero. We visualize the aligned source point using red dots.}
\label{fig:toy-ex}
\end{figure}

\subsection{SNAREseq dataset}

The second dataset we use is the SNAREseq dataset, containing the chromatine accessibility (ATAC-seq) and gene expression (RNA-seq) of 1047 single cells of 4 different types. The source ATAC-seq modality has dimension \(p=19\), while the target RNA-seq modality has dimension \(q=10\).
Again, the task is to match source and target modality datasets using an entropic map from the source to the target. 
In Table \ref{tab:full_SNAREseq} we report the results of \(\mathsf{C}\Ri_r\mathsf{UOT}\) on the full SNAREseq dataset when varying the parameter \(\lambda\) and the same type of results are reported in Table \ref{tab:sub_SNAREseq} for the randomly subsampled SNAREseq dataset. 
For the experiment with the subsampled SNAREseq dataset, we randomly pick two cell types: in the source dataset we subsample them at \(50 \%\) and the other two types at \(75 \%\); in the target dataset we subsample them at \(75 \%\) and the other two types at \(50 \%\).

\begin{table}
    \centering
    \begin{minipage}{0.3\textwidth}
        \centering
        \begin{tabular}{l c}
            \toprule
            $\lambda$ & LTA \\
            \midrule
            $+\infty$ & \textbf{0.944}  \\
            5.0       & \textbf{0.944} \\
            2.5       & 0.941 \\
            1.0       & 0.941  \\
            0.5       & 0.938 \\
            \bottomrule
        \end{tabular}
        \caption{Full SNAREseq dataset results.}
        \label{tab:full_SNAREseq}
    \end{minipage}
    \hfill
    \begin{minipage}{0.6\textwidth}
        \centering
        \includegraphics[width=\linewidth]{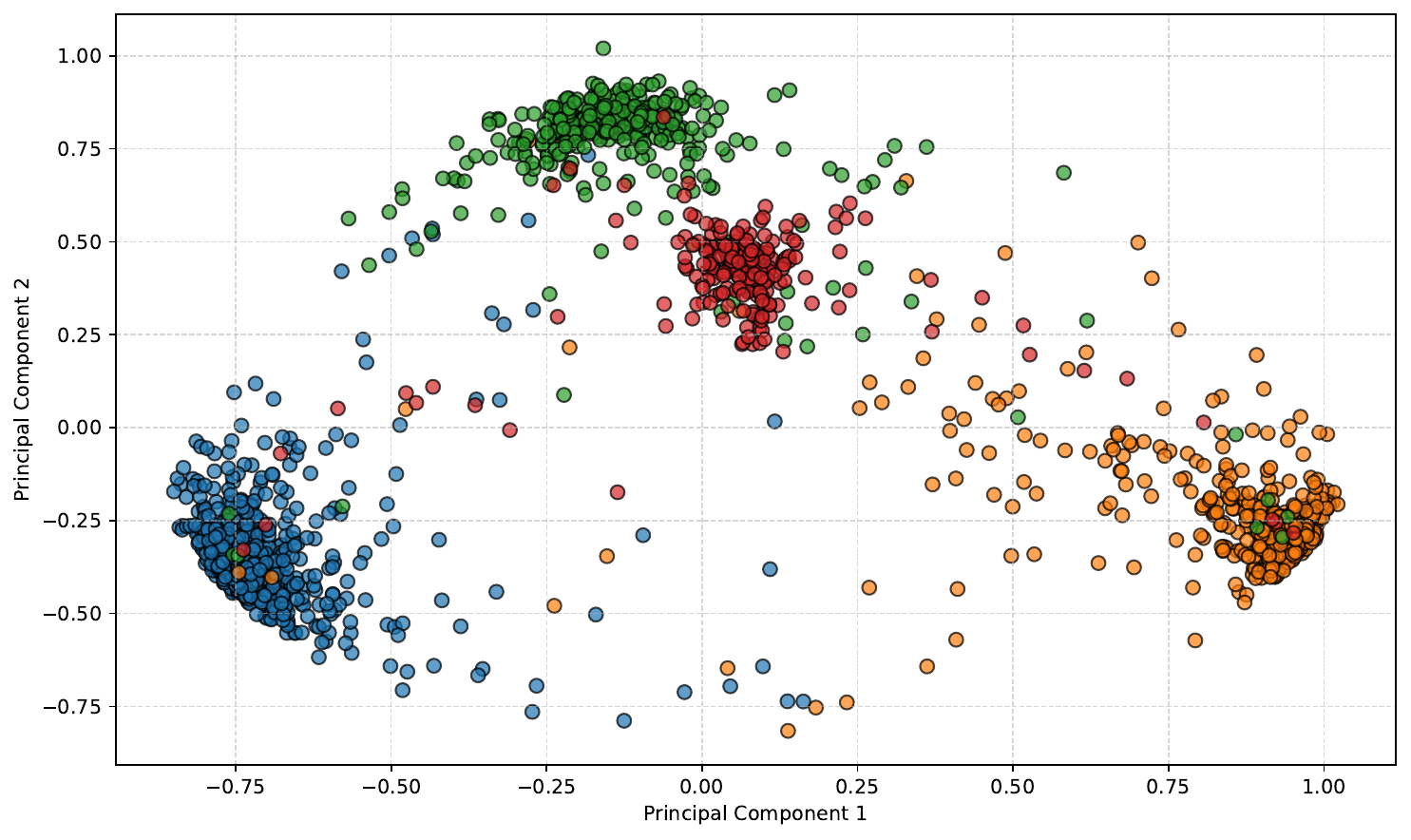}
        \captionof{figure}{Visualization of the entropic map alignment of the full SNAREseq dataset with $\lambda=5.0$ using two-dimensional PCA. Different colours refer to different cell types.}
        \label{fig:alignment_full_SNAREseq}
    \end{minipage}
\end{table}

\begin{table}
    \centering
    \begin{minipage}{0.3\textwidth}
        \centering
        \begin{tabular}{l c}
            \toprule
            $\lambda$ & LTA \\
            \midrule
            $+\infty$ & 0.582  \\
            1.0       & 0.656  \\
            0.5       & 0.695  \\
            0.1       & 0.752  \\
            0.07      & \textbf{0.761} \\
            \bottomrule
        \end{tabular}
        \caption{Subsampled SNAREseq dataset results.}
        \label{tab:sub_SNAREseq}
    \end{minipage}
    \hfill
    \begin{minipage}{0.6\textwidth}
        \centering
        \includegraphics[width=\linewidth]{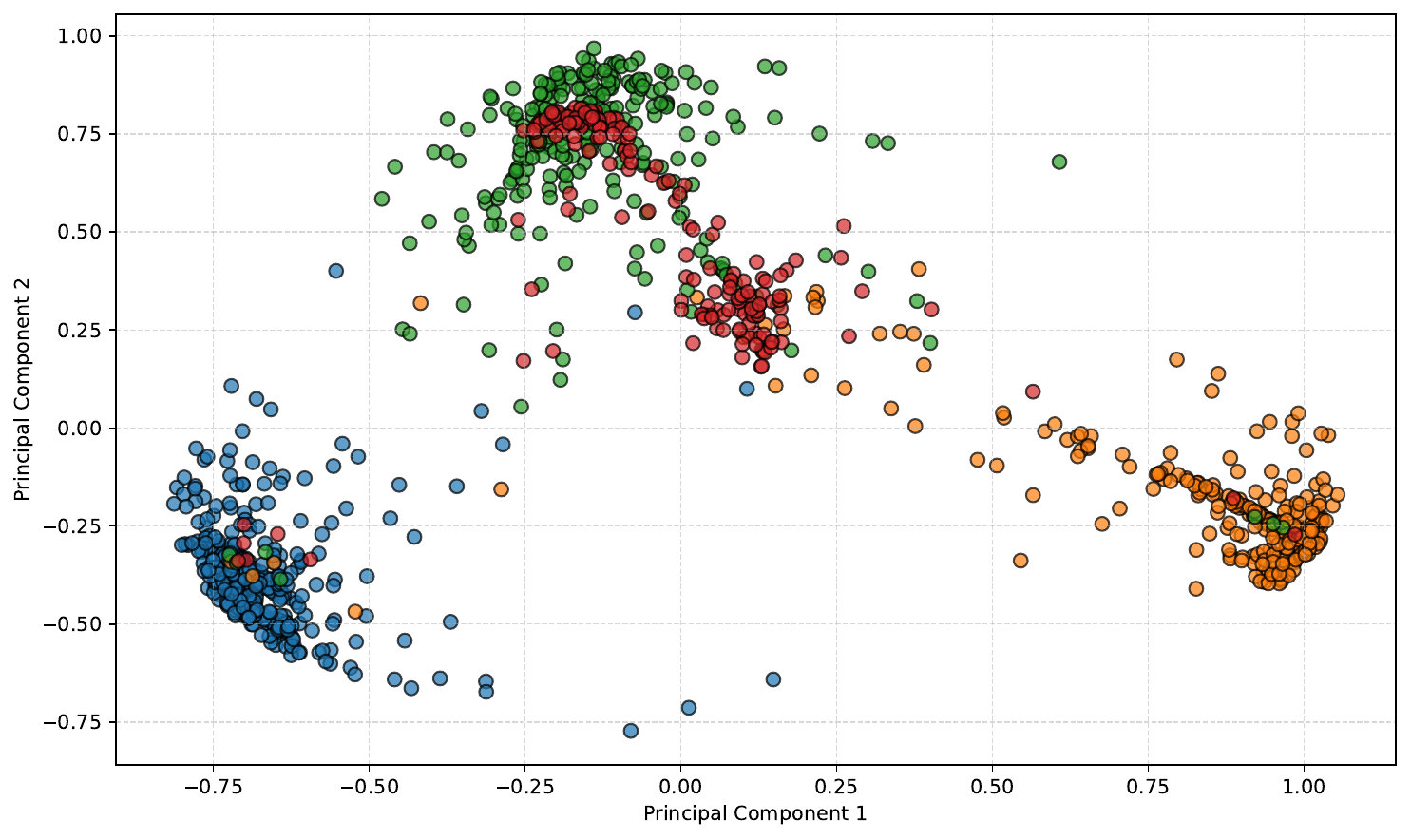}
        \captionof{figure}{Visualization of the entropic map alignment of the subsampled SNAREseq dataset with $\lambda=0.07$ using two-dimensional PCA. Different colours refer to different cell types.}
        \label{fig:alignment_sub_SNAREseq}
    \end{minipage}
\end{table}

\subsection{Details on the entropic map in the case where $M^{*}$ is not surjective}

The low-rank (or sparse) regularizations extend naturally to the cost-regularized
unbalanced optimal transport problem with inner-product cost.
Specifically, for costs of the form
\[
c_{M}(x,y)=-\langle Mx,\,y\rangle
\qquad\text{and}\qquad
\mathcal{R}(M)=\tfrac{1}{2}\norm{M}_F^2+\lambda\,g(M),
\]
where \(g\) encodes a structure constraint (e.g., nuclear norm, 
\(\ell_1\), or \(\ell_{1,2}\)-penalty, or an explicit factorization 
\(M=M_2^\top M_1\)),
the corresponding \(\mathsf{\mathcal{R}UOT}\) problem
\[
\inf_{M,\pi\ge 0}
\int_{\X\times\Y} -\langle Mx, y\rangle\,\mathrm{d}\pi
+\mathcal{R}(M)
+\D{\varphi_1}{\pi_1}{\alpha}
+\D{\varphi_2}{\pi_2}{\beta}
+\varepsilon\D{\mathrm{KL}}{\pi}{\rho}
\]
admits minimizers under the same hypotheses as in
Theorem~\ref{app:existence-RUOT}.
For fixed \(M\), the minimization over \(\pi\) is precisely the
entropic \(\mathsf{UOT}\) with cost \(c_M\),
solvable via the unbalanced Sinkhorn algorithm.
For fixed \(\pi\), the update in \(M\) is a proximal step on
\(\int yx^\top\,\mathrm{d}\pi\) and takes the same closed form
as in the balanced case.

\vspace{2mm}\noindent
\textbf{On Monge maps.}
When \(\alpha\) is absolute constinuous with respect to the  and the learned linear operator
\(M^*\) is \emph{surjective} (i.e. of full column rank \(q\)),
the assumptions of Theorem~\ref{thm:monge-CR} apply and
the optimal coupling \(\pi^*\) is induced by a Monge map
\[
\pi^* = (\mathrm{id},T_*)_\#\pi_1^*, 
\qquad T_* = -\nabla f_*\circ M^*,
\]
where \(f_*\) is the Kantorovich potential associated with the
inner-product cost between \(M^*_\#\pi_1^*\) and \(\pi_2^*\).
If the regularizer \(g(M)\) enforces a low-rank structure
(\(\mathrm{rank}(M^*)=r<q\)),
then \(M^*\) is not surjective and Monge maps are no longer
guaranteed to exist globally.
In this case, one may still interpret the optimal plan as acting on
the lower-dimensional image measure
\(\mu^* = M^*_\#\pi_1^*\), through a map
\(\tilde T: \operatorname{Im}(M^*) \to \Y\)
optimal for the cost \(c_{\mathrm{ip}}(y',y)=-\langle y',y\rangle\),
and write formally
\[
\pi^* = (\mathrm{id},\,\tilde T\circ M^*)_\#\pi_1^*.
\]
The theoretical guarantees of Monge map require the full-rank
assumption on \(M^*\), while the low-rank and sparse parametrizations
remain fully valid from the optimization and numerical perspectives. In Figure \ref{fig:output} we observe how low-rank affects the optimal transport plan across different levels of unbalancedness. Each subplot shows the learned transport map $M_\varepsilon \alpha$ (green),
    the ground-truth map $M_\ast\alpha$ (teal),
    and the target samples $\beta$ (red). Orange lines represent barycentric displacements induced by the optimal plan.When $\lambda \to \infty$ (upper-left plot), the problem reduces to the balanced and the model transports the entire source mass.
    As $\lambda$ decreases ($3.0 \rightarrow 1.5 \rightarrow 0.5$),
    the marginal penalty weakens, allowing partial mass creation or removal:
    the transport plan concentrates on geometrically consistent regions
    while ignoring unmatched components.
    The rightmost plot reports the total transported mass, which decreases monotonically with $\lambda$,
    confirming the progressive relaxation of the mass constraint. Although from Figure \ref{fig:output}, we can see that the learned map aligns closely with the ground-truth low-rank map we need to investigate further statistical guarantees of the learned transport map in case $M^{*}$ is not full rank.

\begin{figure}[ht!]
    \centering
    \includegraphics[width=\linewidth]{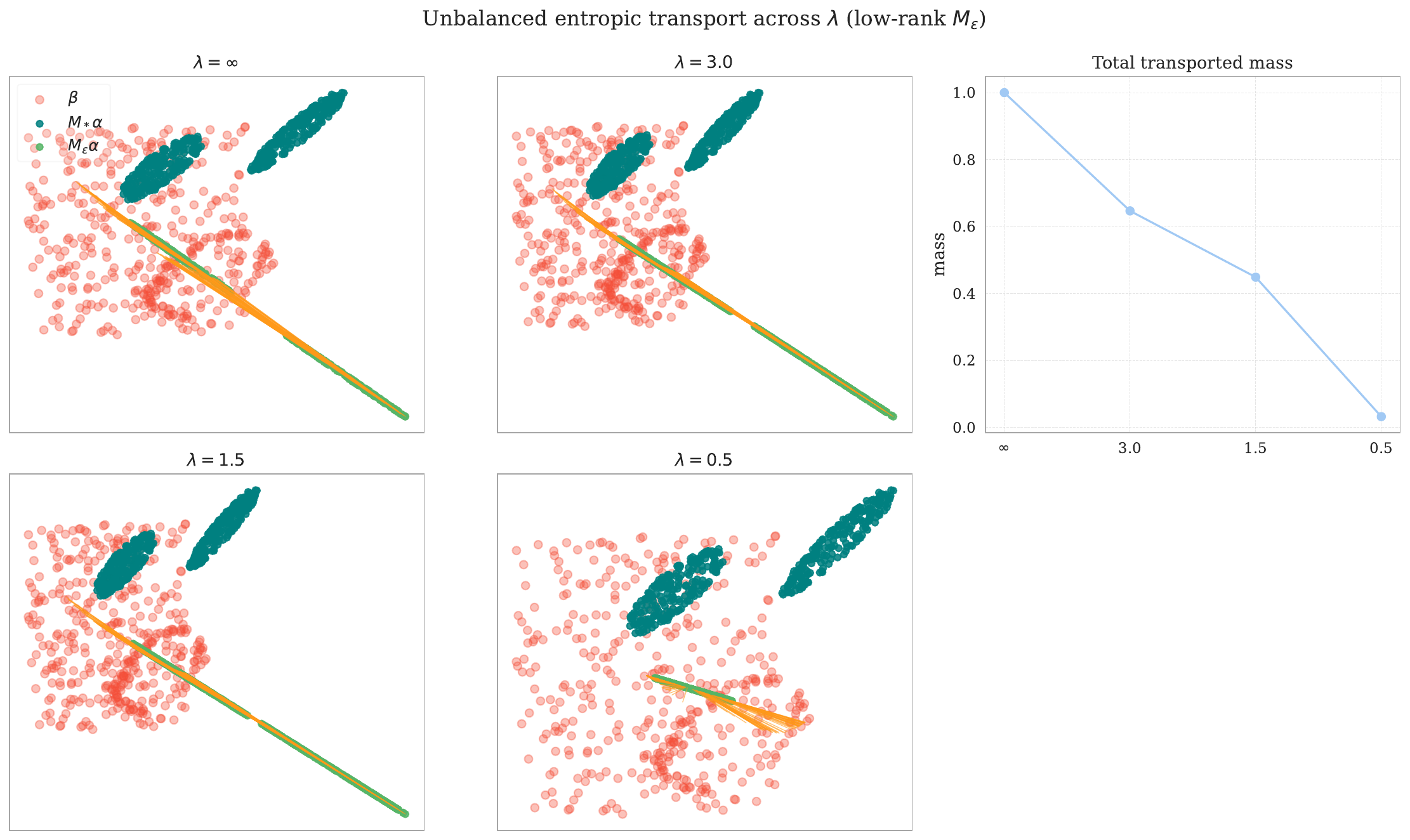}
    \caption{
    \textbf{Low-rank unbalanced optimal transport across unbalancedness levels $\lambda$.}
    The learned map $M_\varepsilon\alpha$ (green) approaches
    the ground-truth $M_{\star}\alpha$ (teal) while ignoring unmatched
    mass in the target $\beta$ (red)as $\lambda$ decreases.
    The total transported mass (right figure) decreases monotonically,
    reflecting the transition from balanced to strongly unbalanced transport.
    }
    \label{fig:output}
\end{figure}

\vfill

\bibliographystyle{unsrt}
\bibliography{bibliography}

@book{sturm2023space,
  title={The space of spaces: curvature bounds and gradient flows on the space of metric measure spaces},
  author={Sturm, K.},
  volume={290},
  number={1443},
  year={2023},
  publisher={American Mathematical Society}
}

@article{chowdhury2019gromov,
  title={The Gromov--Wasserstein distance between networks and stable network invariants},
  author={Chowdhury, S. and M{\'e}moli, F.},
  journal={Information and Inference: A Journal of the IMA},
  volume={8},
  number={4},
  pages={757--787},
  year={2019},
  publisher={Oxford University Press}
}

@article{peyre2019computational,
  title={Computational optimal transport: With applications to data science},
  author={Peyr{\'e}, G. and Cuturi, M.},
  journal={Foundations and Trends{\textregistered} in Machine Learning},
  volume={11},
  number={5-6},
  pages={355--607},
  year={2019},
  publisher={Now Publishers, Inc.}
}

@article{carlier2017convergence,
  title={Convergence of entropic schemes for optimal transport and gradient flows},
  author={Carlier, G. and Duval, V. and Peyr{\'e}, G. and Schmitzer, B.},
  journal={SIAM Journal on Mathematical Analysis},
  volume={49},
  number={2},
  pages={1385--1418},
  year={2017},
  publisher={SIAM}
}

@article{vayer2020contribution,
  title={A contribution to optimal transport on incomparable spaces},
  author={Vayer, T.},
  journal={arXiv preprint arXiv:2011.04447},
  year={2020}
}

@article{liero2018optimal,
  title={Optimal entropy-transport problems and a new Hellinger--Kantorovich distance between positive measures},
  author={Liero, M. and Mielke, A. and Savar{\'e}, G.},
  journal={Inventiones mathematicae},
  volume={211},
  number={3},
  pages={969--1117},
  year={2018},
  publisher={Springer}
}

@book{santambrogio2015optimal,
  title={Optimal transport for applied mathematicians},
  author={Santambrogio, F.},
  volume={87},
  year={2015},
  publisher={Springer}
}

@inproceedings{sebbouh2024structured,
  title={Structured transforms across spaces with cost-regularized optimal transport},
  author={Sebbouh, O. and Cuturi, M. and Peyr{\'e}, G.},
  booktitle={International Conference on Artificial Intelligence and Statistics},
  pages={586--594},
  year={2024},
  organization={PMLR}
}

@article{nutz2021introduction,
  title={Introduction to entropic optimal transport},
  author={Nutz, M.},
  journal={Lecture notes, Columbia University},
  year={2021}
}

@book{villani2008optimal,
  title={Optimal transport: old and new},
  author={Villani, C.},
  volume={338},
  year={2008},
  publisher={Springer}
}

@article{sejourne2019sinkhorn,
  title={Sinkhorn divergences for unbalanced optimal transport},
  author={S{\'e}journ{\'e}, T. and Feydy, J. and Vialard, F. and Trouv{\'e}, A. and Peyr{\'e}, G.},
  journal={arXiv preprint arXiv:1910.12958},
  year={2019}
}

@article{pooladian2021entropic,
  title={Entropic estimation of optimal transport maps},
  author={Pooladian, A. and Niles-Weed, J.},
  journal={arXiv preprint arXiv:2109.12004},
  year={2021}
}

@book{klenke2008probability,
  title={Probability theory: a comprehensive course},
  author={Klenke, A.},
  year={2008},
  publisher={Springer}
}

@article{bertsekas1997nonlinear,
  title={Nonlinear programming},
  author={Bertsekas, D.},
  journal={Journal of the Operational Research Society},
  volume={48},
  number={3},
  pages={334--334},
  year={1997},
  publisher={Taylor \& Francis}
}

@phdthesis{argelaguet2021statistical,
  title={Statistical methods for the integrative analysis of single-cell multi-omics data},
  author={Argelaguet, R.},
  year={2021}
}

@inproceedings{montavon2016wasserstein,
  title={Wasserstein training of restricted Boltzmann machines},
  author={Montavon, G. and M{\"u}ller, K. and Cuturi, M.},
  booktitle={Advances in Neural Information Processing Systems},
  year={2016}
}

@inproceedings{genevay2018learning,
  title={Learning generative models with Sinkhorn divergences},
  author={Genevay, A. and Peyr{\'e}, G. and Cuturi, M.},
  booktitle={Proceedings of the Twenty-First International Conference on Artificial Intelligence and Statistics (AISTATS)},
  pages={1608--1617},
  year={2018}
}

@inproceedings{tong2023improving,
  title={Improving and generalizing flow-based generative models with minibatch optimal transport},
  author={Tong, A.},
  booktitle={International Conference on Machine Learning (Workshop Track)},
  year={2023}
}

@inproceedings{sinha2018certifying,
  title={Certifying some distributional robustness with principled adversarial training},
  author={Sinha, A. and Namkoong, H. and Duchi, J.},
  booktitle={International Conference on Learning Representations (ICLR)},
  year={2018}
}

@inproceedings{wong2019wasserstein,
  title={Wasserstein adversarial examples via projected Sinkhorn iterations},
  author={Wong, E. and Schmidt, L. and Kolter, J Z.},
  booktitle={International Conference on Machine Learning (ICML)},
  year={2019}
}

@inproceedings{courty2017joint,
  title={Joint distribution optimal transportation for domain adaptation},
  author={Courty, N. and Flamary, R. and Tuia, D. and Rakotomamonjy, A.},
  booktitle={Advances in Neural Information Processing Systems},
  year={2017}
}

@article{janati2020multi,
  title={Multi-subject dictionary learning and Wasserstein means for alignment of medical data},
  author={Janati, H. and Cuturi, M. and Gramfort, A.},
  journal={Advances in Neural Information Processing Systems},
  year={2020}
}

@article{schiebinger2019optimal,
  title={Optimal-transport analysis of single-cell gene expression identifies developmental trajectories in reprogramming},
  author={Schiebinger, G.},
  journal={Cell},
  volume={176},
  number={4},
  pages={928--943},
  year={2019}
}

@article{bunne2023learning,
  title={Learning single-cell multimodal data with optimal transport},
  author={Bunne, C.},
  journal={Nature Biotechnology},
  year={2023}
}

@article{memoli2011gromov,
  title={Gromov--Wasserstein distances and the metric approach to object matching},
  author={M{\'e}moli, F.},
  journal={Foundations of Computational Mathematics},
  volume={11},
  number={4},
  pages={417--487},
  year={2011}
}

@article{dumont2022existence,
  title={Existence and uniqueness for Gromov--Wasserstein optimal transport},
  author={Dumont, L. and Papadakis, N. and Peyr{\'e}, G.},
  journal={SIAM Journal on Mathematical Analysis},
  volume={54},
  number={2},
  pages={2022--2059},
  year={2022}
}

@article{mazelet2025unsupervisedlearningoptimaltransport,
      title={Unsupervised Learning for Optimal Transport plan prediction between unbalanced graphs}, 
      author={Mazelet, S. and Flamary, R. and Thirion B.},
      year={2025},
    archivePrefix={arXiv},
}

@inproceedings{xu2019gromov,
  title={Gromov-wasserstein learning for graph matching and node embedding},
  author={Xu, H. and Luo, D. and Zha, H. and Duke, L.},
  booktitle={International conference on machine learning},
  pages={6932--6941},
  year={2019},
  organization={PMLR}
}

@inproceedings{titouan2019optimal,
  title={Optimal transport for structured data with application on graphs},
  author={Vayer, T. and Courty, N. and Tavenard, R. and Flamary, R.},
  booktitle={International Conference on Machine Learning},
  pages={6275--6284},
  year={2019},
  organization={PMLR}
}

@article{thual2022aligning,
  title={Aligning individual brains with fused unbalanced Gromov Wasserstein},
  author={Thual, A. and Tran, Q. and Zemskova, T. and Courty, N. and Flamary, R. and Dehaene, S. and Thirion, B.},
  journal={Advances in neural information processing systems},
  volume={35},
  pages={21792--21804},
  year={2022}
}

@article{demetci2022scot,
  title={SCOT: single-cell multi-omics alignment with optimal transport},
  author={Demetci, P. and Santorella, R. and Sandstede, B. and Noble, W. and Singh, R.},
  journal={Journal of computational biology},
  volume={29},
  number={1},
  pages={3--18},
  year={2022},
  publisher={Mary Ann Liebert, Inc., publishers 140 Huguenot Street, 3rd Floor New~…}
}

@article{demetci2022scotv2,
  title={Scotv2: Single-cell multiomic alignment with disproportionate cell-type representation},
  author={Demetci, P. and Santorella, R. and Chakravarthy, M. and Sandstede, B. and Singh, R.},
  journal={Journal of Computational Biology},
  volume={29},
  number={11},
  pages={1213--1228},
  year={2022},
  publisher={Mary Ann Liebert, Inc., publishers 140 Huguenot Street, 3rd Floor New~…}
}

@article{tseng2001convergence,
  title={Convergence of a block coordinate descent method for nondifferentiable minimization},
  author={Tseng, P.},
  journal={Journal of optimization theory and applications},
  volume={109},
  number={3},
  pages={475--494},
  year={2001},
  publisher={Springer}
}

@inproceedings{brechet2023critical,
  title={Critical points and convergence analysis of generative deep linear networks trained with Bures-Wasserstein loss},
  author={Br{\'e}chet, P. and Papagiannouli, K. and An, J. and Mont{\'u}far, G.},
  booktitle={International Conference on Machine Learning},
  pages={3106--3147},
  year={2023},
  organization={PMLR}
}

@inproceedings{fatras2021unbalanced,
  title={Unbalanced minibatch optimal transport; applications to domain adaptation},
  author={Fatras, K. and S{\'e}journ{\'e}, T. and Flamary, R. and Courty, N.},
  booktitle={International conference on machine learning},
  pages={3186--3197},
  year={2021},
  organization={PMLR}
}

@article{bonet2024slicing,
  title={Slicing Unbalanced Optimal Transport},
  author={Bonet, C. and Nadjahi, K. and S{\'e}journ{\'e}, T. and Fatras, K. and Courty, N.},
  journal={Transactions on Machine Learning Research},
  year={2024}
}

@article{Kondratyev2016,
  title={A new optimal transport distance on the space of finite Radon measures},
  author={Kondratyev, S. and Monsaingeon, L. and Vorotnikov, D.},
  journal={Advances in Differential Equations},
  volume={21},
  number={11/12},
  pages={1117--1164},
  year={2016}
}

@article{Liero2018,
  title={Optimal transport in competition with reaction: The Hellinger--Kantorovich distance and geodesic curves},
  author={Liero, M. and Mielke, A. and Savar{\'e}, G.},
  journal={SIAM Journal on Mathematical Analysis},
  volume={48},
  number={4},
  pages={2869--2911},
  year={2016}
}

@article{Chizat2018,
  title={Scaling algorithms for unbalanced optimal transport problems},
  author={Chizat, L. and Peyr{\'e}, G. and Schmitzer, B. and Vialard, F.},
  journal={Mathematics of Computation},
  volume={87},
  number={314},
  pages={2563--2609},
  year={2018}
}

@article{Rotskoff2019,
  title={Trainability and accuracy of neural networks: An interacting particle system approach},
  author={Rotskoff, G. and Vanden-Eijnden, E.},
  journal={Communications on Pure and Applied Mathematics},
  volume={72},
  number={6},
  pages={1262--1298},
  year={2019}
}

@article{Schiebinger2019,
  title={Optimal-transport analysis of single-cell gene expression identifies developmental trajectories in reprogramming},
  author={Schiebinger, G. and Shu, Jenny and Tabaka, Marcin and Cleary, Brian and Subramanian, Vidya and Solomon, Andrew and Gould, Jesse and Liu, Stanley and Lin, Shirley and Berube, Pierre and others},
  journal={Cell},
  volume={176},
  number={4},
  pages={928--943},
  year={2019}
}

@inproceedings{Fatras2021,
  title={Unbalanced minibatch optimal transport; applications to domain adaptation},
  author={Fatras, K. and Zine, Y. and Flamary, R; and Gribonval, R. and Courty, N;},
  booktitle={ICML},
  pages={3186--3197},
  year={2021}
}

@article{Pele2009,
  title={Fast and robust earth mover’s distances},
  author={Pele, O. and Werman, M.},
  journal={2009 IEEE 12th International Conference on Computer Vision},
  pages={460--467},
  year={2009}
}

@article{Dudley1969,
  title={The speed of mean Glivenko-Cantelli convergence},
  author={Dudley, R.},
  journal={The Annals of Mathematical Statistics},
  volume={40},
  number={1},
  pages={40--50},
  year={1969}
}

@inproceedings{Cuturi2013,
  title={Sinkhorn distances: Lightspeed computation of optimal transport},
  author={Cuturi, M.},
  booktitle={NeurIPS},
  pages={2292--2300},
  year={2013}
}

@article{Pham2020,
  title={Unbalanced optimal transport with entropic regularization},
  author={Pham, M. and Gozlan, N.},
  journal={arXiv preprint arXiv:2009.04269},
  year={2020}
}

@inproceedings{Fatras2020,
  title={Minibatch optimal transport; computational and statistical aspects},
  author={Fatras, K. and Courty, N. and Flamary, R. and Cuturi, M.},
  booktitle={NeurIPS},
  pages={1--12},
  year={2020}
}

@article{cheow2016single,
  title={Single-cell multimodal profiling reveals cellular epigenetic heterogeneity},
  author={Cheow, L. and Courtois, E. and Tan, Y. and Viswanathan, R. and Xing, Q. and Tan, Rui Z. and Tan, D. SW and Robson, P. and Loh, Y. and Quake, S.},
  journal={Nature methods},
  volume={13},
  number={10},
  pages={833--836},
  year={2016},
  publisher={Nature Publishing Group US New York}
}

@article{chen2019high,
  title={High-throughput sequencing of the transcriptome and chromatin accessibility in the same cell},
  author={Chen, S. and Lake, B. and Zhang, K.},
  journal={Nature biotechnology},
  volume={37},
  number={12},
  pages={1452--1457},
  year={2019},
  publisher={Nature Publishing Group US New York}
}

@article{eyring2023unbalancedness,
  title={Unbalancedness in neural monge maps improves unpaired domain translation},
  author={Eyring, Luca and Klein, Dominik and Uscidda, Th{\'e}o and Palla, Giovanni and Kilbertus, Niki and Akata, Zeynep and Theis, Fabian},
  journal={arXiv preprint arXiv:2311.15100},
  year={2023}
}

\end{document}